\documentclass{article}
\usepackage{arxiv}

\usepackage[utf8]{inputenc}
\usepackage[T1]{fontenc}
\usepackage{graphicx}
\usepackage{subcaption}
\usepackage{booktabs}
\usepackage{bm}
\usepackage{bbm}
\usepackage{tcolorbox}
\usepackage{multicol,multirow}
\usepackage{natbib}
\usepackage{comment}
\usepackage{booktabs}

\usepackage{enumerate}   
\usepackage{amsmath}
\usepackage{tikz}

\usetikzlibrary{positioning,arrows}

\usepackage{wrapfig}

\usepackage{amsmath}
\usepackage{appendix}
\usepackage{amssymb}

\usepackage{amsthm}
\usepackage{hyperref}
\usepackage[capitalise]{cleveref}
\Crefname{assumption}{Assumption}{Assumptions}

\theoremstyle{plain}
\newtheorem{theorem}{Theorem}

\newtheorem{corollary}{Corollary}
\theoremstyle{definition}
\newtheorem{assumption}{Assumption}
\newtheorem{proposition}{Proposition}
\newtheorem{definition}{Definition}
\newtheorem{remark}{Remark}

\def\Cramer{Cram\'{e}r}

\newcommand{\pa}{\mathrm{\pa}}

\newcommand{\epol}{\pi^\mathrm{e}}

\renewcommand{\eqref}[1]{(\ref{#1})}

\newcommand{\RN}[1]{%
  \textup{\uppercase\expandafter{\romannumeral#1}}%
}

\usepackage{color,graphicx,lscape,longtable}

\def\boxit#1{\vbox{\hrule\hbox{\vrule\kern6pt\vbox{\kern6pt#1\kern6pt}\kern6pt\vrule}\hrule}}

\usepackage{xcolor}
\newcount\Comments  
\newcommand{\kibitz}[2]{\ifnum\Comments=1\textcolor{#1}{#2}\fi}

\usepackage{algorithm}
\usepackage{algorithmic}

\usepackage{braket}

\title{Confidence Interval for Off-Policy Evaluation\\
from Dependent Samples via Bandit Algorithm:\\
Approach from Standardized Martingales}

\author{Masahiro Kato\\
  CyberAgent Inc.\\
  Tokyo, Japan\\
  \texttt{masahiro\_kato@cyberagent.co.jp}
}

\begin{document}

\maketitle

\begin{abstract}
This study addresses the problem of \emph{off-policy evaluation} (OPE) from \emph{dependent samples} obtained via the \emph{bandit algorithm}. The goal of OPE is to evaluate a new policy using historical data obtained from \emph{behavior policies} generated by the bandit algorithm. Because the bandit algorithm updates the policy based on past observations, the samples are not \emph{independent and identically distributed} (i.i.d.). However, several existing methods for OPE do not take this issue into account and are based on the assumption that samples are i.i.d. In this study, we address this problem by constructing an estimator from a \emph{standardized martingale difference sequence}. To standardize the sequence, we consider using \emph{evaluation data} or \emph{sample splitting} with \emph{two-step estimation}. This technique produces an estimator with asymptotic normality without restricting a class of behavior policies. In an experiment, the proposed estimator performs better than existing methods, which assume that the behavior policy converges to a time-invariant policy.
\end{abstract}

\section{Introduction}
The \emph{multi-armed bandit} (MAB) problem is one of the sequential decision-making problems, which is applied in various applications, such as ad-optimization, personalized medicine, search engines, and recommendation systems. In the MAB problem, we choose an action at a period and observe the outcome. To reduce the cost of searching for better policies, the evaluation of a new policy using historical data obtained from past trials has gathered great attention \citep{kdd2009_ads, www2010_cb, AtheySusan2017EPL}. This framework is called \emph{off-policy evaluation} (OPE) \citep{dudik2011doubly,wang2017optimal,narita2018,pmlr-v97-bibaut19a,Kallus2019IntrinsicallyES,Oberst2019}. Although several methods for OPE have been proposed, existing studies often presume that the samples are \emph{independent and identically distributed} (i.i.d.). However, in the MAB problem algorithm, the policy is usually updated based on past observations, and the samples are not i.i.d. due to the policy updating process. In this case, the consistency and asymptotic normality of the existing methods are not guaranteed. In particular, the asymptotic normality is critical because it guarantees the $\sqrt{n}$-consistency and is needed for the confidence interval in hypothesis testing to determine whether the new policy is better than the existing policy. Thus, the motivation for establishing a novel method for performing OPE from dependent samples is strong.  

Several previous studies consider OPE from dependent samples \citep{Laan2008TheCA,1911.02768,KallusNathan2019EBtC,Kato2020}. The strategies for deriving asymptotic normality can be classified into three approaches. In the first approach, under the assumption that the policy used in past trials converges to the time-invariant policy in probability, the asymptotic normality is derived using the theories related to \emph{martingales} \citep{Laan2008TheCA,1911.02768,Kato2020}. In the second approach, the policy is assumed to be batch updated, where, although the policy is updated using past observations, the sample size under a fixed policy is sufficient \citep{Hahn2011}. In the third approach, both the stationarity and the conditions of \emph{mixingales} \citep{KosorokMichaelR2008ItEP}, which requires the independence of time-separated samples, are assumed \citep{KallusNathan2019EBtC}. However, existing methods have several drawbacks. With the first approach, the assumption that the policy converges to a time-invariant policy sometimes does not hold. For example, if the algorithms of the MAB problem change during past trials, it is difficult to justify this assumption. The second and third approaches cannot be applied in cases where the policy is sequentially updated, and the conditions of mixingales do not hold, respectively. 

To overcome these drawbacks, in this study, we propose a novel estimator with the asymptotic normality, which does not assume policy convergence, the existence of batches, and the independence of time-separated samples. To accomplish this goal, we refocus on the first approach. In the first approach, we use a \emph{martingale difference sequence} (MDS) for constructing an estimator. In this approach, the convergence of the policy used in past trials to the time-invariant policy in probability is assumed for using \emph{central limit theorem} (CLT) of an MDS, which requires that the variance of an MDS becomes asymptotically constant. In this paper, instead of assuming the convergence of the policy, we consider constructing an MDS with asymptotically constant variance by \emph{standardization} using an estimator of the variance. To implement the standardization, we assume access to evaluation (test) data or conduct \emph{sample splitting}. Through this technique, we can guarantee the asymptotic normality of the proposed estimator with fewer assumptions.

\paragraph{Contributions and Organization of this Paper:} This paper has three main contributions. First, this paper provides a theoretical solution for causal inference from dependent samples obtained via bandit feedback, which is more realistic than the situation that has been often considered in the existing work. Second, the proposed estimator achieves the asymptotic normality with fewer assumptions compared with existing solutions. Third, the estimator also achieves a lower mean squared error (MSE) in experiments using benchmark datasets. 

\paragraph{Additional Related Work:}
There are various studies for OPE under the assumption that samples are i.i.d. \citep{dudik2011doubly,wang2017optimal,narita2018,pmlr-v97-bibaut19a,Kallus2019IntrinsicallyES,Oberst2019}, but there are fewer studies focus on the case in which samples are not i.i.d. When the policy converges, \citet{Laan2008TheCA}, \citet{1911.02768}, and \citet{Kato2020} proposed using an estimator based on an MDS. When including the covariates for OPE, \citet{Laan2008TheCA} mainly suggested using targeted maximum likelihood estimation. On the other hand, \citet{1911.02768} and \citet{Kato2020} proposed using an estimator based on the doubly robust or augmented IPW estimator. \citet{KallusNathan2019EBtC} dealt with the application of reinforcement learning and proposed an estimator constructed from dependent samples. In their study, they proposed \emph{time cross-fitting}, which is a variant of cross-fitting of double/debiased machine learning \citep{ChernozhukovVictor2018Dmlf}. The theoretical guarantee is based on mixingale theory. For obtaining the asymptotic normality, conducting the standardization is also proposed by \citet{Luedtke2016}. \citet{Luedtke2016} tries to solve a similar problem in the batch policy updating. In addition to this difference of policy updating, there are the following two points different from us. First, in their solution, they require estimators of $f^*$ and $e^*$ to satisfy some conditions like the Donsker condition. However, in our estimator, we do not impose such restrictions on estimators of $f^*$ and $e^*$ except for the boundedness. This difference is critical because, in practice, as \citet{ChernozhukovVictor2018Dmlf} pointed out, the estimators of $f^*$ and $e^*$ usually do not satisfy the Donsker condition when we use modern methods such as the Lasso and Random forest regression. Second, for obtaining an estimator with the asymptotic normality under sequential policy updating, we propose using evaluation data or sample splitting. 


\section{Problem Setting}
\label{sec:prob_setting}
In this section, we introduce our problem setting.

\subsection{Date Generating Process}
\label{sec:dgp}
Let $A_t$ be the action taking variable in $\mathcal{A}=\{1,2,\dots,K\}$, $X_t$ be the \emph{covariate} observed by the decision maker when choosing an action, and $\mathcal{X}$ be the domain of covariate. Let us denote a random variable of a reward at a period $t$ as a function $Y_t:\mathcal{A}\to\mathbb{R}$. Let $\mathbbm{1}[\cdot]$ be an indicator function. In this paper, we have access to a set of \emph{historical data}, $\mathcal{S}_T=\{(X_t, A_t, Y_t)\}^{T}_{t=1}$, with the following data generating process (DGP):
\begin{align}
\label{eq:DGP}
\big\{(X_t, A_t, Y_t)\big\}^{T}_{t=1}\sim p(x)p_t(a\mid x, \Omega_{t-1})p(y\mid a, x),
\end{align}
where $Y_t = \sum^K_{a=1}\mathbbm{1}[A_t=a]Y_t(a)$, $p(x)$ denote the density of the covariate $X_t$, $p_t(a\mid x, \Omega_{t-1})$ denote the probability of assigning an action $A_t$ conditioned on a covariate $X_t$, $p(y\mid a, x)$ denote the density of an outcome $Y_t$ conditioned on $A_t$ and $X_t$, and $\Omega_{t-1}\in \mathcal{M}_{t-1}$ denotes the history defined as $\Omega_{t-1}=\{X_{t-1}, A_{t-1}, Y_{t-1}, \dots, X_{1}, A_1, Y_{1}\}$ with the domain $\mathcal{M}_{t-1}$.
We assume that $p(x)$ and $p(y\mid a, x)$ are invariant across periods, but $p_t(a\mid x, \Omega_{t-1})$ can take different value across periods. Besides, we allow the decision maker to change $p_t(a\mid x, \Omega_{t-1})$ based on past observations. In this case, samples $\big\{(X_t, A_t, Y_t)\big\}^{T}_{t=1}$ are correlated across periods, i.e., samples are not i.i.d. Here, we introduce \emph{behavior policies}, which determine the probability $p_t(a\mid x, \Omega_{t-1})$. Let a behavior policy $\pi_t:\mathcal{A}\times\mathcal{X}\times\mathcal{M}_{t-1}\to(0,1)$ in the $t$-th period be a function of a covariate $X_t$, an action $A_t$, and a history $\Omega_{t-1}$. In this study, we assume that $\pi_t(a\mid x, \Omega_{t-1}) = p_t(a\mid x, \Omega_{t-1})$. 

\begin{remark}[Evaluation Data] In Section~\ref{sec:ope_with_unknown_var}, we also assume that there is access to \emph{evaluation data}, $\mathcal{E}_N = \big\{X_i\big\}^{N}_{i=1}$. In Remark~\ref{rem:sample_splitting}, we relax this assumption by introducing the \emph{sample splitting}.
\end{remark}

\subsection{Off-Policy Evaluation}
\label{sec:opeopl}
Under the DGP defined in the previous subsection, we consider estimating the value of an \emph{evaluation policy} using samples obtained under the behavior policies. Let an evaluation policy $\epol:\mathcal{A}\times\mathcal{X}\to(0,1)$ be a function of a covariate $X_t$ and an action $A_t$; this evaluation policy can be considered as the probability of taking the action $A_t$ conditional on the covariate $X_t$. We are interested in estimating the expected reward from a given pre-specified evaluation policy $\pi^{\mathrm{e}}(a \mid x)$. Then, we define the expected reward under $\epol$ as $R(\epol) := \mathbb{E}\left[\sum^{K}_{a=1}\epol(a, X_t)Y_t(a)\right]$. We also denote $R(\epol)$ as $\theta_0$. The goal of this study is to estimate $R(\pi^{\mathrm{e}})$ using dependent samples under a sequentially-updated policy. To identify $R(\pi^{\mathrm{e}})$, we make the following assumptions on the policies and the outcomes. 

\begin{assumption}
\label{asm:knowldege_pol}
The behavior policies $\pi_t(a\mid \cdot, \Omega_{t-1})$ and the evaluation policy $\epol(a\mid \cdot)$ are known\footnote{The proposed method requires the function of the polices, not only $\pi_t(a\mid X_t, \Omega_{t-1})$ for a specific $X_t$.}. Additionally, the evaluation policy is deterministic.
\end{assumption}
\begin{assumption}\label{asm:DGP}
There exist $C_1$ and $C_2$ such that $\frac{\epol(a\mid x)}{\pi_t(a\mid x, \Omega_{t-1})}\leq C_1$ and $|Y_t| \leq C_2$.
\end{assumption}

The deterministic evaluation policy looks restrictive. However, the optimal policy that maximizes the expected reward is a deterministic policy when the policy is not restricted. Therefore, when we find the optimal policy from a set of policies, such as \emph{off-policy learning} \citep{ZhaoYingqi2012EITR,KitagawaToru2018WSBT,ZhouZhengyuan2018OMPL,Chernozhukov2019}, this assumption can be accommodated. Thus, restricting the policy class to the deterministic policy class would be reasonable in practice. In fact, most of the existing methods follow this way. 

\begin{remark}[Existing Methods for OPE]\label{rem:exist_OPE}
We review three types of standard estimators of $R(\pi^{\mathrm{e}})$ under the case where $p_1(a\mid x)=p_2(a\mid x)=\cdots=p_T(a\mid x)=p(a\mid x)$ in the DGP defined in \eqref{eq:DGP}. The first estimator is an inverse probability weighting (IPW) estimator given by $\frac{1}{T}\sum^T_{t=1}\sum^K_{a=1}\frac{\epol(a\mid X_t)\mathbbm{1}[A_t=a]Y_t}{p(a\mid X_t)}$ \citep{Horvitz1952,rubin87,hirano2003efficient}. Although this estimator is unbiased when the behavior policy is known, it has a high variance. The second estimator is a direct method (DM) estimator $\frac{1}{T}\sum^T_{t=1}\sum^K_{a=1}\hat{f}_{T}(a, X_t)$, where $\hat{f}_{T}(a, X_t)$ is an estimator of $f^*(a, X_t)$ \citep{HahnJinyong1998OtRo}. This estimator is known to be weak against model misspecification for $f^*(a, X_t)$. The third estimator is a doubly robust estimator defined as $\frac{1}{T}\sum^T_{t=1}\sum^K_{a=1}\left(\frac{\epol(a\mid X_t)\mathbbm{1}[A_t=a]\big(Y_t - \hat{f}_{T}(a, X_t)\big)}{p(a\mid X_t)} + \epol(a\mid X_t)\hat{f}_{T}(a, X_t)\right)$ \citep{robins94,ChernozhukovVictor2018Dmlf}. Under certain conditions, it is known that this estimator achieves the lower bound (a.k.a semiparametric lower bound), which is the lower bound of the asymptotic MSE of OPE, among regular $\sqrt{T}$-consistent estimators \citep[Theorem 25.20]{VaartA.W.vander1998As}. 
\end{remark}

\begin{remark}[Unconfoundedness]
Existing methods often make an assumption called unconfoundedness, that is, they assume that the outcomes $(Y_t(0), Y_t(1), \dots, Y_t(K))$ and the action $A_t$ are conditionally independent on $X_t$ for identification. In the DGP, we can obtain a similar result from the assumption that an action $A_t$ is chosen according to probability $p_t(a\mid x, \Omega_{t-1})$.
\end{remark}

\begin{remark}[Semiparametric Lower Bound]
\label{rem:semi_low}
The lower bound of the variance can be defined for an estimator under some posited models of the DGP. If this posited model is a parametric model, then the lower bound is equal to the \Cramer-Rao lower bound. When this posited model is a non- or semiparametric model, the corresponding lower bound can still be defined \citep{bickel98}. \citet{narita2018} shows that the semiparametric lower bound of the DGP~\eqref{eq:DGP} under $p_1(a\mid x, \Omega_{0})=p_2(a\mid x, \Omega_{1})=\cdots=p_T(a\mid x, \Omega_{T-1})=p(a\mid x)$ is $\mathbb{E}\left[\sum^{K}_{a=1}\left\{\frac{\big(\epol(a\mid X_t)\big)^2v^*(a, X_t)}{p(a\mid X_t)} + \Big(\epol(a\mid X_t)f^*(a, X_t) - \theta_0\Big)^2\right\}\right]$.
\end{remark}

\paragraph{Notations:} Let $a$ be an action in $\mathcal{A}$. We denote $\mathbb{E}[Y_t(a)\mid x]$, $\mathbb{E}[Y^2_t(a)\mid x]$, and $\mathrm{Var}(Y_t(a)\mid x)$ as $f^*(a, x)$, $e^*(a, x)$, and $v^*(a, x)$, respectively. Let $\mathcal{F}$ be the domain of $f^*(a, x)$. Let $\hat{f}_{t}(a, x)$, $\hat e_{t}(a, x)$, and $\hat{\theta}_t$ be estimators of $f^*(a, x)$, $e^*(a, x)$, and $\theta_0$ constructed from $\Omega_{t}$, respectively. Let $\mathcal{N}(\mu, \mathrm{var})$ be the normal distribution with mean $\mu$ and variance $\mathrm{var}$.

\subsection{Policy Updating}
In the DGP defined in Section~\ref{sec:dgp}, we allow the behavior policy $\pi_t(a \mid x, \Omega_{t-1})$ to be updated based on past observations. For example, when applying some algorithms proposed in the MAB problem, we usually optimize the probability to maximize the cumulative reward using past observations. When the behavior policy $\pi_t(a \mid x, \Omega_{t-1})$ depends on past observations, the samples are not i.i.d. In this case, we cannot apply the standard methods for OPE, which assume that the samples are i.i.d. To clarify this problem, we classify the patterns of policy updating into two scenarios: \emph{sequential policy updating} and \emph{batch policy updating}. In sequential policy updating, the policy can be updated each period. In batch policy updating, the policy is updated after observing some samples \citep{Hahn2011,narita2018}. In this study, we consider a problem using sequential policy updating. 

\section{OPE under a Converging Policy}
In this section, we review OPE when a policy converges to the time-invariant policy in probability. As \citet{Laan2008TheCA}, \citet{1911.02768} and \citet{Kato2020} proposed, in such a situation, we can construct an estimator with the asymptotic normality by using an MDS \citep{hall2014martingale}. To conduct OPE when the behavior policy is known, we can use the following score function $\psi:\Theta\times\mathcal{X}\times\mathcal{A}\times\mathbb{R}\times\Pi_t\times\Pi\times\mathcal{F}\to\mathbb{R}$ \citep{ChernozhukovVictor2018Dmlf}:
\footnotesize
\begin{align}
\label{score}
\psi_t\big(\theta; x, k, y, \pi_t, \epol, f\big) = \sum^{K}_{a=1}\left\{\frac{\epol(a\mid x)\mathbbm{1}[A_t=a]\big\{y - f(a, x)\big\}}{\pi_{t}(a\mid x, \Omega_{t-1})} + \epol(a\mid x)f(a, x)\right\} - \theta,
\end{align}
\normalsize
where $\Theta$, $\Pi_t$, and $\Pi$ are spaces of a parameter $\theta$, a behavior policies $\pi_t$, and an evaluation policy $\epol$. Then, for the true policy value $\theta_0$, the score function with $\theta = \theta_0$, $\psi_t\big(\theta_0; X_t, A_t, Y_t, \pi_t, \epol, \hat{f}_{t - 1}\big)$, satisfies $\mathbb{E}[\psi_t\big(\theta_0; X_t, A_t, Y_t, \pi_t, \epol, \hat{f}_{t - 1}\big)] = 0$. As we explain in Remark~\ref{rem:exist_OPE}, there are several other candidate estimators, including the IPW estimator. However, the estimator based on the score function defined in \eqref{score} achieves the lowest asymptotic variance when $\pi_t$ is fixed among periods \citep{bickel98,scharfstein99}. Using $\big\{\psi_t\big(\theta_0; X_t, A_t, Y_t, \pi_t, \epol, \hat{f}_{t - 1}\big)\big\}^T_{t=1}$, we consider constructing an MDS $\big\{z_t\big\}^T_{t=1}$ that satisfies the conditions of the following CLT for an MDS.
\begin{proposition}
\label{prp:marclt}[CLT for an MDS, \citet{GVK126800421}, Proposition~7.9, p.~194] Let $\{z_t\}^\infty_{t=1}$ be an MDS. Suppose that (a) $\mathbb{E}\big[z^2_t\big] = \nu^2_t > 0$ with $\big(1/T\big) \sum^T_{t=1}\nu^2_t\to \nu^2 > 0$; (b) $\mathbb{E}\big[|z_t|^r\big] < \infty$ for some $r>2$; (c) $\big(1/T\big)\sum^T_{t=1}z^2_t\xrightarrow{\mathrm{p}} \nu^2$. Then $\sqrt{T}\frac{1}{T}\sum^T_{t=1}z_t\xrightarrow{\mathrm{d}}\mathcal{N}(0, \nu^2)$.
\end{proposition}
When $Y_t$, $\epol/\pi_t$, and $f_{t - 1}$ are bounded, we can easily show that the condition~(b) holds (Assumption~\ref{asm:DGP}). The remaining task is to show that conditions (a) and (c) hold. 

\citet{Kato2020} constructed $\ddot{z}_t = \psi_t\big(\theta_0; X_t, A_t, Y_t, \pi_t, \epol, \hat{f}_{t - 1})$, where $\left\{\ddot{z}_t\right\}^T_{t=1}$ is an MDS. Then, by assuming $\pi_t(a\mid x, \Omega_{t-1})\xrightarrow{\mathrm{p}}\alpha(a\mid x)$ and some regularity conditions, \citet{Kato2020} proved that the sequence $\left\{\ddot{z}_t\right\}^T_{t=1}$ satisfies the conditions~(a) and (c) with $\nu^2 = \mathbb{E}\left[\left(\sum^{K}_{a=1}\left\{\frac{\epol(a\mid x)\mathbbm{1}[A_t=a]\big\{Y_t - f^*(a, X_t)\big\}}{\alpha(a\mid X_t)} + \epol(a\mid X_t)f^*(a, X_t)\right\} - \theta_0\right)^2\right]$. Thus, by the assumption $\pi_t(a\mid x, \Omega_{t-1})\xrightarrow{\mathrm{p}}\alpha(a\mid x)$, the variance of $h_t$ converges to time-invariant value, and this property is the motivation of the assumption proposed by \citet{Kato2020}. Then, under $\pi_t(a\mid x, \Omega_{t-1})\xrightarrow{\mathrm{p}}\alpha(a\mid x)$, the estimator proposed defined as $\hat{\theta}^{\mathrm{A2IPW}}_T = \frac{1}{T}\sum^T_{t=1}\sum^{K}_{a=1}\left\{\frac{\epol(a\mid X_t)\mathbbm{1}[A_t=a]\big\{Y_t - \hat{f}_{t-1}(a, X_t)\big\}}{\pi_{t}(a\mid X_t, \Omega_{t - 1})} + \epol(a\mid x)\hat{f}_{t-1}(a, X_t)\right\}$ can be proved to be consistent and has the asymptotic normality \citep{1911.02768,Kato2020}\footnote{They assume the almost sure convergence of the policy for showing the mean convergence of the variance, but we can relax it to the convergence in probability by the boundedness of the MDS \citep{loeve1977probability}.}. We refer to this estimator as \emph{Adaptive AIPW} (A2IPW). We can also define the corresponding IPW estimator and refer to it as \emph{adaptive IPW} (AdaIPW) estimator.

\begin{remark}[Efficiency of A2IPW]
The asymptotic variance of the A2IPW estimator matches the semiparametric lower bound under a time-invariant policy $\alpha$ \citep{Kato2020}.
\end{remark}

\section{OPE from Dependent Samples under Sequential Policy Updating}
\label{sec:sequential}
In this section, we discuss a strategy for constructing an estimator without assuming the convergence of the policy. Then, we show a method for constructing an estimator with the asymptotic normality.

\subsection{Strategy for OPE from Dependent Samples under Sequential Policy Updating}
However, there are various applications in which we cannot assume that $\pi_t(a\mid x, \Omega_{t-1})\xrightarrow{\mathrm{p}}\alpha(a\mid x)$. Therefore, we have a strong motivation for constructing an estimator with asymptotic normality without assuming that $\pi_t(a\mid x, \Omega_{t-1})\xrightarrow{\mathrm{p}}\alpha(a\mid x)$. Our strategy is to construct an MDS $\{\tilde{z}_t\}^T_{t=1}$ satisfying conditions (a) and (c) by standardization of $\ddot{z}_t$ as $\tilde{z}_t = \left(\sqrt{\mathrm{Var}\big(\ddot{z}_t\mid \Omega_{t-1}\big)} \right)^{-1}\ddot{z}_t$, where $\mathrm{Var}\left(\ddot{z}_t\mid \Omega_{t-1}\right)$ is the conditional variance of $\ddot{z}_t$. The standardization means that the variance of $\tilde{z}_t$ becomes $1$. Then, we can easily show that the conditions~(a) and (b) hold. Thus, if we know the variance $\mathrm{Var}\big(\ddot{z}_t\mid \Omega_{t-1}\big)$, we can construct an MDS with constant variance by the standardization. 

\begin{remark}[Lindeberg Condition]
The proposed strategy can be interpreted using the Lindeberg condition. For all $\epsilon > 0$ and an MDS $\{z_t\}^T_{t=1}$ such that $\sum^T_{t=1}\mathbb{E}\big[z^2_t \mid \Omega_{t-1}\big]/\sum^T_{t=1}\mathbb{E}\big[z^2_t\big] \xrightarrow{\mathrm{p}} 1$, the Lindeberg condition is $\left(\sum^T_{t=1}\mathbb{E}\big[z^2_t\big]\right)^{-1}\sum^T_{t=1}\mathbb{E}\left[z^2_t\mathbbm{1}\left[|z_t|\geq \epsilon \sqrt{\sum^T_{t=1}\mathbb{E}\big[z^2_t\big]}\right]\right]\xrightarrow{\mathrm{p}} 0$ as $T\to\infty$ \citep{brown1971}. The first condition $\sum^T_{t=1}\mathbb{E}\big[z^2_t \mid \Omega_{t-1}\big]/\sum^T_{t=1}\mathbb{E}\big[z^2_t\big] \xrightarrow{\mathrm{p}} 1$ restricts the class of an MDS to an MDS with asymptotically time-invariant variance. Here, we note that $\sigma^2_t = \mathbb{E}\big[z^2_t \mid \Omega_{t-1}\big]$. The strategy of this paper can be intuitively interpreted as a transformation of the original sequence $\ddot{z}_t$ into such an MDS with asymptotically time-invariant variance.
\end{remark}

\subsection{OPE with a Weighted Average Estimator: Case with Known Variance}
Then, we propose an estimator based on the MDS $\big\{\tilde{z}_t\big\}^T_{t=1}$ and refer to the estimator as a \emph{weighted Average Adaptive Augmented IPW} (A3IPW) estimator. For $\big\{\hat{f}_t\big\}^{T-1}_{t=0}$ and $\sigma^2_t = \mathrm{Var}\big(\ddot{z}_t\mid \Omega_{t-1}\big)$, we define the A3IPW estimator as $\hat{\theta}^{\mathrm{A3IPW}}_T = \left(\sum^T_{t=1}\frac{1}{\sqrt{\sigma^2_t}}\right)^{-1}\sum^T_{t=1}\frac{1}{\sqrt{\sigma^2_t}}\sum^{K}_{a=1}\left\{\frac{\epol(a\mid X_t)\mathbbm{1}[A_t=a]\big\{Y_t - \hat{f}_{t-1}(a, X_t)\big\}}{\pi_{t}(a\mid X_t, \Omega_{t - 1})} + \epol(a\mid x)\hat{f}_{t-1}(a, X_t)\right\}$.
Then, we can show that the estimator has the asymptotic normality as follows.
\begin{theorem}
\label{thm:asymptotic_distribution_a3ipw}
Suppose that there exists $C_2$ such that $|\hat{f}_t|\leq C_2$. Then, under Assumption~\ref{asm:DGP}, $\left(\sum^T_{t=1}\frac{1}{\sqrt{\sigma^2_t}}\right)\sqrt{T}\left(\hat{\theta}^{\mathrm{A3IPW}}_{T} - \theta_0\right) \xrightarrow{\mathrm{d}} \mathcal{N}\left(0, 1\right)$ as $T\to\infty$.
\end{theorem}
\begin{proof}
For $\tilde{z}_t =\frac{1}{\sqrt{\sigma^2_t}}\left(\sum^{K}_{a=1}\left\{\frac{\epol(a\mid X_t)\mathbbm{1}[A_t=a]\big\{y - \hat{f}_{t-1}(a, X_t)\big\}}{\pi_{t}(a\mid X_t, \Omega_{t - 1})} + \epol(a\mid x)\hat{f}_{t-1}(a, X_t)\right\} - \theta_0\right)$, $\big\{\tilde{z}_t\big\}^{T}_{t=1}$ is an MDS because $\mathbb{E}\left[\tilde{z}_t\mid \Omega_{t-1}\right] = \frac{1}{\sqrt{\sigma^2_t}}\mathbb{E}\left[\ddot{z}_t\mid \Omega_{t-1}\right] = 0$. Because $\mathrm{Var}\left(\ddot{z}_t/\sqrt{\sigma^2_t}\right) = \mathbb{E}\big[\mathbb{E}[(\ddot{z}_t - \mathbb{E}[\ddot{z}])^2 \mid \Omega_{t-1}]/\sigma^2_t\big] = \mathbb{E}\big[\mathbb{E}[\ddot{z}^2_t \mid \Omega_{t-1}]/\sigma^2_t\big] = \mathbb{E}\big[\mathrm{Var}\big(\ddot{z}_t \mid \Omega_{t-1}\big)/\sigma^2_t\big] =1$, we have $\big(1/T\big) \sum^T_{t=1}\mathrm{Var}\left(\tilde{z}_t\right)= 1$. Here, we used $\mathbb{E}[\ddot{z}_t \mid \Omega_{t-1}]=\mathbb{E}[\ddot{z}_t]=0$. Besides, we can show that $\tilde{z}_t$ is bounded from assumptions, and $\big(1/T\big)\sum^T_{t=1}\tilde{z}^2_t \xrightarrow{p} 1$ from the weak law of large numbers of Proposition~\ref{prp:mrtgl_WLLN} in Appendix~\ref{appdx;prelim}. Then, we can apply CLT for an MDS. 
\end{proof}

\subsection{OPE with a Weighted Average Estimator: Case with Unknown Variance}
\label{sec:ope_with_unknown_var}
In general, the variance $\sigma^2_t$ is unknown and needed to be replaced with an estimator. However, we cannot use the samples used for constructing $\pi_t$ and $\hat{f}_{t-1}$ to estimate $\sigma^2_t$ due to the dependency. For this reason, we assume access to a dataset only with covariates, $\mathcal{E}_N = \big\{X_i\big\}^{N}_{i=1}$, which is independent from the historical data $\mathcal{S}_T$. Because $\mathcal{E}_N$ is independent from $\mathcal{S}_T = \big\{(X_t, A_t, Y_t)\big\}^{T}_{t=1}$, we can approximate $\sigma^2_t$ by applying the law of large numbers to the sample average $g'_{t, N}$, defined as
\begin{align*}
\frac{1}{N} \sum^N_{i = 1} \sum^{K}_{a=1}\left\{\frac{\big(\epol(a\mid X_i)\big)^2\big(\hat{e}_{t-1}(a, X_i) - \hat{f}^2_{t-1}(a, X_i)\big)}{\pi_{t}(a\mid X_i, \Omega_{t - 1})} + \Big(\epol(a\mid X_i)\hat{f}_{t-1}(a, X_i) - \hat{\theta}_{t-1} \Big)^2\right\}.
\end{align*}
The consistency of $g'_{t, N}$ to the variance of $\psi_t\big(\theta_0; X_t, A_t, Y_t, \pi_t, \epol, f^*\big)$ is showed as follows.
\begin{theorem}
\label{thm:consist_var}
Suppose that, for all $x\in\mathcal{X}$ and $a\in\mathcal{A}$, $\hat{f}_{t-1}(a, x)-f^*(a, x)\xrightarrow{\mathrm{p}}0$, $\hat{e}_{t-1}(a, x)-e^*(a, x)\xrightarrow{\mathrm{p}}0$ and $\hat{\theta}_t\xrightarrow{\mathrm{p}}\theta_0$ as $t\to\infty$. Then, for $\pi_t\in\Pi$, we have $g'_{t, (1-r)T} - \sigma^{*2}_t\xrightarrow{\mathrm{p}}0$ as $t\to\infty$ and $N\to\infty$, where $\sigma^{*2}_t = \mathbb{E}\left[\sum^{K}_{a=1}\left\{\frac{\big(\epol(a\mid X_t)\big)^2\nu^*(a, X_t)}{\pi_{t}(a\mid X_t, \Omega_{t - 1})} + \Big(\epol(a\mid X_t)f^*(a, X_t)-\theta_0\Big)^2\right\}\right]$.
\end{theorem}

In practice, we need to avoid having the estimator of $\sigma^{*2}_t$ be less than or equal to $0$. Thus, instead of $g'_{t, (1-r)T}$, we use the estimator $g_{t, N}$, defined as $g_{t, (1-r)T}=\max\big\{g'_{t, N}, \epsilon\big\}$, where $\epsilon > 0$ is the lower bound of the variance $\sigma^{*2}_t$. The estimator $g_{t, N}$ is consistent, as is $g'_{t, N}$. The value $\epsilon$ is assumed to be known, and we can use any sufficiently small value for $\epsilon$ because it is simply a technical term. The proof is shown in Appendix~\ref{appdx:consist_var}. Then, using $g_{t, N}$, we define the \emph{Feasible A3IPW} (FA3IPW) estimator $\hat{\theta}^{\mathrm{FA3IPW}}_{T}$ as $\left(\sum^{T}_{t=1}\frac{1}{\sqrt{g_{t, N}}}\right)^{-1}\sum^{T}_{t=1}\sum^{K}_{a=1}\frac{1}{\sqrt{g_{t, N}}}\left\{\frac{\epol(a\mid X_t)\mathbbm{1}[A_t=a]\big\{Y_t - \hat{f}_{t-1}(a, X_t)\big\}}{\pi_{t}(a\mid X_t, \Omega_{t - 1})} + \epol(a\mid x)\hat{f}_{t-1}(a, X_t)\right\}$. Similarly, we can also define the corresponding IPW estimator and refer to it as \emph{Feasible weighted Average Adaptive IPW estimator} (FA2daIPW). 

In some applications, we can obtain the dataset $\mathcal{E}_N = \big\{X_i\big\}^{N}_{i=1}$ as evaluation (test) data, which is the target of a new policy \citep{kato_uehara_2020}. On the other hand, by introducing a sample splitting, we can relax the assumption as the following remark.

\begin{remark}[Sample Splitting] 
\label{rem:sample_splitting}
To relax the access to $\mathcal{E}_N$, we consider splitting $\big\{(X_t, A_t, Y_t)\big\}^{T}_{t=1}$ into $\big\{(X_t, A_t, Y_t)\big\}^{\lfloor rT\rfloor}_{t=1}$ and $\big\{(X_t, A_t, Y_t)\big\}^{T}_{t=\lfloor rT\rfloor + 1}$, where $0 < r < 1$. In the following discussion, for simplicity, we assume that $rT$ is an integer (i.e., $\lfloor rT\rfloor=rT$). The samples $\big\{X_t\big\}^{T}_{t=rT}$ and $\big\{(X_t, A_t, Y_t)\big\}^{rT}_{t=1}$ are independent. Thus, to approximate $\sigma^2_t$, we can apply the law of large numbers to the sample average $g_{t, (1-r)T} = \max\left\{g'_{t, (1-r)T}, \epsilon\right\}$, where $g'_{t, (1-r)T}$ is defined as $\frac{1}{(1-r)T} \sum^T_{s = rT} \sum^{K}_{a=1}\left\{\frac{\big(\epol(a\mid X_s)\big)^2\big(\hat{e}_{t-1}(a, X_s) - \hat{f}^2_{t-1}(a, X_s)\big)}{\pi_{t}(a\mid X_s, \Omega_{t - 1})} + \Big(\epol(a\mid X_s)\hat{f}_{t-1}(a, X_s) - \hat{\theta}_{t-1} \Big)^2\right\}$. More details on sample splitting are provided in Appendix~\ref{appdx:sample_splitting}.
\end{remark}

\subsection{Asymptotic Properties of a FA3IPW Estimator}
Let $\hat{z}_t = \sum^{T}_{t=1}\frac{1}{\sqrt{g_{t, N}}}\left(\sum^{K}_{a=1}\left\{\frac{\epol(a\mid X_t)\mathbbm{1}[A_t=a]\big\{Y_t - \hat{f}_{t-1}(a, X_t)\big\}}{\pi_{t}(a\mid X_t, \Omega_{t - 1})} + \epol(a\mid X_t)\hat{f}_{t-1}(a, X_t)\right\} - \theta_0\right)$. We can show that $\hat{z}_t$ is an MDS as $\mathbb{E}\left[\hat{z}_t\mid \Omega_{t-1}\right] = \frac{1}{\sqrt{g_{t, N}}}\mathbb{E}\left[\ddot{z}_t\mid \Omega_{t-1}\right] = 0$.
Based on this property, the following theorems gives us the consistency and asymptotic normality of $\hat{\theta}^{\mathrm{FA3IPW}}_{T}$. First, we show the consistency of $\hat{\theta}^{\mathrm{FA3IPW}}_{T}$, i.e., $\hat{\theta}^{\mathrm{FA3IPW}}_{T} \xrightarrow{\mathrm{p}} \theta_0$.
\begin{theorem}[Consistency of the Proposed Estimator]
\label{thm:consitency} Suppose that there exist $C_2$ and $C_3$ such that $|\hat{f}_t|\leq C_2$ and $0 < 1/\sqrt{g_{t, N}} \leq C_3$. Then, under Assumptions~\ref{asm:DGP}, $\hat{\theta}^{\mathrm{FA3IPW}}_{T} \xrightarrow{\mathrm{p}} \theta_0$. 
\end{theorem}
\begin{proof} From the boundedness, we can apply the law of large numbers for an MDS (Proposition~\ref{prp:mrtgl_WLLN} in Appendix~\ref{appdx;prelim}). Therefore, $\frac{1}{T}\sum^{T}_{t=1}\hat{z}_t \xrightarrow{\mathrm{p}} 0$.
\end{proof}
The following theorem shows the asymptotic normality of $\hat{\theta}^{\mathrm{FA3IPW}}_T$. The proof is in Appendix~\ref{appdx:proof_thm:main}.
\begin{theorem}[Asymptotic Distribution of the Proposed Estimator]
\label{thm:main}
Suppose that
\begin{description}
\item[(i)] there exist $C_2$ and $C_3$ such that $|\hat{f}_t|\leq C_2$ and $0 < 1/\sqrt{g_{t, N}} \leq C_3$;
\item[(ii)] Pointwise convergence in probability of $\hat{f}_t$, i.e., $\forall x \in \mathcal{X}$, $\hat{f}_t(a, x)- f^*(a, x)\xrightarrow{\mathrm{p}}0$ as $t\to \infty$;
\item[(ii)] $g_{t, N}- \sigma^{*2}_t\xrightarrow{\mathrm{p}}0$ as $t\to \infty$ and $N\to \infty$.
\end{description}
Then, under Assumptions~\ref{asm:DGP}, $\left(\frac{1}{\sqrt{T}}\sum^{T}_{t=1}\frac{1}{\sqrt{g_{t, N}}}\right)\left(\hat{\theta}^{\mathrm{FA3IPW}}_{T} - \theta_0\right) \xrightarrow{\mathrm{d}} \mathcal{N}\left(0, 1\right)$ as $T, N\to \infty$.
\end{theorem} 

\begin{remark}[Difference of Assumption between the Consistency and Asymptotic Normality]
The asymptotic normality requires $g_{t, N}- \sigma^{2*}_t\xrightarrow{\mathrm{p}}0$. However, we can obtain the consistency $\hat{\theta}^{\mathrm{A3IPW}}_T \xrightarrow{\mathrm{p}} \theta_0$ without assuming $g_{t, N}- \sigma^{*2}_t\xrightarrow{\mathrm{p}}0$. From this property, we can obtain the consistent estimator of the variance, $\sigma^{*2}_t$, which requires the consistent estimator of the policy value, $\theta_0$. 
\end{remark}

\begin{remark}[Heuristic Stabilization]
\label{rem:heur}
The proposed estimator requires a sequence of $\big\{f_t\big\}^{T-1}_{t=0}$. Because the early estimator of $f_t$ may be unstable due to the insufficient amount of training data, some heuristic stabilizations may improve its performance. \citet{1911.02768} proposed a similar estimator and suggested using a weight sequence for stabilization. In the setting of \citet{Kato2020}, they can change $p_t(a\mid x, \Omega_{t-1})$ and proposed adjusting the probability to stabilize their proposed estimator. In this study, we take a simple approach by suggesting discarding some samples at first and constructing the proposed estimator without them. We call this estimator the \emph{stabilized FA3IPW} (SFA3IPW) estimator. In some experiments, this approach improves the estimator's performance. 
\end{remark}

\subsection{Efficiency of A3IPW Estimator}
In the problem setting, it is not easy to discuss an efficiency of an estimator. For simplicity, we consider the lower bound of the variance of an estimator within a class of estimators such that $\hat{\theta}_T = \sum^{T}_{t=1}w_t\tilde{\psi}_t(X_t, A_t, Y_t, \pi_t)$, where $\tilde{\psi}$ is a score function and $w_t$ is a weight satisfying $\sum^T_{t=1}w_t = 1$. Let us note that we usually use $w_t=1/T$ in the case where samples are i.i.d. First, we consider the variance of $\tilde{\psi}_t$. In OPE, the lowest variance is given as $\mathbb{E}\left[\sum^{K}_{a=1}\left\{\frac{\big(\epol(a\mid X_t)\big)^2v^*(a, X_t)}{\pi_t(a\mid X, \Omega_{t-1})} + \Big(\epol(a\mid X_t)f^*(a, X_t) - \theta_0\Big)^2\right\}\right]$. In the proposed estimator, each element of $\{\tilde{\psi}_t\}^{T}_{t=1}$ achieves this variance. In addition, it is well known that for a sequence of positive values $\phi_t = \mathbb{E}[\tilde{\psi}^2_t \mid \Omega_{t-1}]$, the weight $w_t =\frac{1}{\sqrt{\phi_t}}/\sum^{T}_{i=1}\frac{1}{\sqrt{\phi_i}}$ minimizes the variance of $\hat{\theta}_T$ when the covariances among the score functions are $0$. When $\{\phi_t\}^T_{t=1}$ is an MDS, the covariance is $0$. Thus, the proposed estimator uses the weights and scores that minimize the variance for an estimator with a form that satisfies $\sum^{T}_{t=1}w_t\tilde{\psi}_t(X_t, A_t, Y_t, \pi_t)$.

\section{Main Algorithm: Two-step FA3IPW Estimation}
\label{sec:main_algorithm}
Under the assumption of $g_{t, N}- \sigma^{*2}_t\xrightarrow{\mathrm{p}}0$, the proposed estimator $\hat{\theta}^{\mathrm{A3IPW}}_T$ has the asymptotic normality. However, to estimate $\sigma^{*2}_t = \mathrm{Var}\big(\ddot{z}_t\mid \Omega_{t-1}\big)$, a consistent estimator of $\theta_0$ is required, which is why we use $g_{t, N}$. However, we can obtain a consistent estimator of $\theta_0$ without imposing the assumption that $g_{t, N} - \sigma^{*2}_t\xrightarrow{\mathrm{p}} 0$. Based on these properties, we propose \emph{two-step estimation}. First, using arbitrary appropriate values for $g_{t, N} = \tilde{g}_{t}$ for $t=1,2,\dots,T-1$, such as $\tilde{g}_{t}=1$, we obtain the initial estimates, $\left\{\tilde{\theta}_t\right\}^{T-1}_{t=1}$. Then, using $\tilde{\theta}_t \xrightarrow{\mathrm{p}} \theta_0$, we construct $g_{t, N} - \sigma^{*2}_t \xrightarrow{\mathrm{p}} 0$. Finally, we can estimate the next estimator of $\theta_0$, $\hat{\theta}^{\mathrm{TSFA3IPW}}_{T}$, with asymptotic normality from Theorem~\ref{thm:main}. We summarize the process of \emph{two-step FA3IPW} (TSFA3IPW) estimation in Algorithm~\ref{alg}. 

\begin{algorithm}[tb]
   \caption{TSFA3IPW Estimation}
   \label{alg}
\begin{algorithmic}
   \STATE For $t=1,2,\dots, T$, let $q_t = \sum^{K}_{a=1}\left\{\frac{\epol(a\mid X_t)\mathbbm{1}[A_t=a]\big\{Y_t - \hat{f}_{t-1}(a, X_t)\big\}}{\pi_{t}(a\mid X_t, \Omega_{t - 1})} + \epol(a\mid X_t)\hat{f}_{t-1}(a, X_t)\right\}$. 
   \STATE {\bfseries Initialization:} Let $\tilde g_{t}$ be $1$ or some other positive value for $t=1,2,\dots,T-1$.
   \STATE Obtain $\left\{\tilde{\theta}_t\right\}^{T-1}_{t=1}$ using $\left\{\tilde{g}_{t}\right\}^{T-1}_{t=1}$ by $\tilde{\theta}_t = \left(\sum^{t}_{s=1}\frac{1}{\sqrt{\tilde{g}_{s}}}\right)^{-1}\sum^t_{s=1}\frac{1}{\sqrt{\tilde{g}_{s}}}q_s$.
   \STATE Obtain $\left\{g_{t, N}\right\}^{T}_{t=1}$ using $\left\{\tilde{\theta}_t\right\}^{T-1}_{t=1}$, an appropriate value $\hat{\theta}^{(1)}_0$ such as $\hat{\theta}^{(1)}_0 = 0$, and $\mathcal{E}_N$.
   \STATE Obtain $\hat \theta^{\mathrm{TSFA3IPW}}_{T}$ using $\left\{g_{t, N}\right\}^{T}_{t=1}$ by $\hat \theta^{\mathrm{TSFA3IPW}}_{T} = \left(\sum^{T}_{s=1}\frac{1}{\sqrt{g_{s, N}}}\right)^{-1}\sum^{T}_{s=1}\frac{1}{\sqrt{g_{s, N}}}q_s$.
\end{algorithmic}
\end{algorithm}  

\begin{remark}[Estimation of Functions \texorpdfstring{$f^*_t$}{f} and \texorpdfstring{$e^*_t$}{g}]
Several nonparametric regression estimators can be proved to be consistent for dependent samples via certain bandit algorithms \citep{yang2002}.
\end{remark}

\begin{table*}[t]

\begin{center}
\caption{Experimental results under the RW policy. The method with the lowest MSE is in bold.} 
\medskip
\label{tbl:exp_table1}
\scalebox{0.65}[0.65]{
\begin{tabular}{l|rr|rr|rr|rr|rr|rr|rr|rr}
\toprule
Datasets &  \multicolumn{2}{c|}{satimage}& \multicolumn{2}{c|}{pendigits}& \multicolumn{2}{c|}{mnist}& \multicolumn{2}{c|}{dna}& \multicolumn{2}{c|}{letter}& \multicolumn{2}{c|}{sensorless}& \multicolumn{2}{c|}{connect-4}& \multicolumn{2}{c}{covtype} \\
Metrics &      MSE &      SD &      MSE &      SD &      MSE &      SD &      MSE &      SD &      MSE &      SD &     MSE &     SD &      MSE &     SD &     MSE &     SD \\
\hline
FA3IPW &  {\bf 0.037} &  0.193 &  0.064 &  0.254 &  {\bf 0.048} &  0.219 &  0.032 &  0.178 &  0.128 &  0.358 &  0.101 &  0.313 &  {\bf 0.029} &  0.163 &  0.045 &  0.210 \\
SFA3IPW &  {\bf 0.037} &  0.192 &  {\bf 0.061} &  0.243 &  0.051 &  0.225 &  {\bf 0.029} &  0.169 &  0.142 &  0.376 &  {\bf 0.096} &  0.308 &  0.033 &  0.178 &  {\bf 0.042} &  0.206 \\
DM &  0.101 &  0.168 &  0.456 &  0.171 &  0.281 &  0.179 &  0.209 &  0.122 &  0.489 &  0.143 &  0.413 &  0.153 &  0.100 &  0.158 &  0.109 &  0.165 \\
AdaIPW &  0.068 &  0.261 &  0.089 &  0.297 &  0.069 &  0.263 &  0.063 &  0.251 &  0.129 &  0.359 &  0.133 &  0.360 &  0.089 &  0.242 &  0.070 &  0.252 \\
A2IPW &  0.038 &  0.193 &  0.063 &  0.250 &  0.049 &  0.220 &  0.035 &  0.185 &  {\bf 0.119} &  0.345 &  0.100 &  0.314 &  0.032 &  0.164 &  0.052 &  0.226 \\
FA2daIPW &  0.068 &  0.261 &  0.090 &  0.301 &  0.074 &  0.273 &  0.057 &  0.238 &  0.140 &  0.373 &  0.133 &  0.358 &  0.075 &  0.237 &  0.059 &  0.236 \\
\bottomrule
\end{tabular}
} 
\end{center}

\begin{center}
\caption{Experimental results under the UCB policy. The method with the lowest MSE is in bold.} 
\medskip
\label{tbl:exp_table2}
\scalebox{0.65}[0.65]{
\begin{tabular}{l|rr|rr|rr|rr|rr|rr|rr|rr}
\toprule
Datasets &  \multicolumn{2}{c|}{satimage}& \multicolumn{2}{c|}{pendigits}& \multicolumn{2}{c|}{mnist}& \multicolumn{2}{c|}{dna}& \multicolumn{2}{c|}{letter}& \multicolumn{2}{c|}{sensorless}& \multicolumn{2}{c|}{connect-4}& \multicolumn{2}{c}{covtype} \\
Metrics &      MSE &      SD &      MSE &      SD &      MSE &      SD &      MSE &      SD &      MSE &      SD &     MSE &     SD &      MSE &     SD &     MSE &     SD \\
\hline
FA3IPW &  0.041 &  0.201 &  0.095 &  0.308 &  0.098 &  0.306 &  0.026 &  0.160 &  {\bf 0.452} &  0.657 &  0.077 &  0.274 &   3.484 &  1.798 &  0.066 &  0.257 \\
SFA3IPW &  0.042 &  0.198 &  {\bf 0.071} &  0.263 &  0.126 &  0.353 &  0.029 &  0.167 &  0.488 &  0.679 &  {\bf 0.070} &  0.262 &   5.248 &  2.240 &  0.077 &  0.275 \\
DM &  {\bf 0.022} &  0.089 &  0.218 &  0.114 &  0.245 &  0.168 &  0.195 &  0.130 &  0.489 &  0.144 &  0.401 &  0.141 &   {\bf 0.112} &  0.201 &  0.107 &  0.172 \\
AdaIPW &  0.078 &  0.275 &  0.099 &  0.314 &  0.188 &  0.423 &  0.051 &  0.223 &  0.891 &  0.929 &  0.101 &  0.315 &  11.424 &  2.560 &  0.072 &  0.267 \\
A2IPW &  0.043 &  0.206 &  0.084 &  0.290 &  {\bf 0.097} &  0.304 &  {\bf 0.025} &  0.158 &  0.854 &  0.909 &  0.083 &  0.286 &   2.363 &  1.400 &  {\bf 0.060} &  0.245 \\
FA2daIPW &  0.072 &  0.263 &  0.124 &  0.351 &  0.173 &  0.398 &  0.052 &  0.225 &  0.481 &  0.680 &  0.094 &  0.304 &  11.031 &  2.628 &  0.079 &  0.281 \\
\bottomrule
\end{tabular}
} 
\end{center}

\end{table*}

\section{Experiments}
\label{sec:exp}
In this section, using benchmark datasets, we demonstrate the effectiveness of the proposed FA3IPW estimator and the FA3IPW estimator with discarding some early samples for stabilization (SFA3IPW estimator in Remark~\ref{rem:heur}). Following \citet{dudik2011doubly} and \citet{Chow2018}, we evaluate the proposed estimators using classification datasets by transforming them into contextual bandit data. From the LIBSVM repository, we use the satImage, pendigits, mnist, dna, letter, sensorless, connect-4, and covtype datasets \footnote{\url{https://www.csie.ntu.edu.tw/~cjlin/libsvmtools/datasets/}}. First, we create an adaptive policy $\pi^a_t$ by following a random walk and perform upper confidence bound (UCB) algorithms in the MAB problem. Then, we construct a behavior policy as a mixture of $\pi^a_t$ and the uniform random policy $\pi^u_t$, defined as $\pi_t = 0.7\pi^a_t+0.3\pi^u_t$ \citep{Kallus2019IntrinsicallyES}. When we use the RW policy, the policy $\pi^a_t$ does not converge. In contrast, when we use the UCB policy to generate $\pi^a_t$, the policy converges to a time-invariant policy. To construct an evaluation policy, we create a deterministic policy $\pi_d$ by training a logistic regression classifier on historical data. Then, we construct the evaluation policy $\epol$ as a mixture of $\pi^d$ and the uniform random policy $\pi^u$, defined as $\epol = 0.7\pi^d+0.3\pi^u$. Through experiments, the behavior policy $\pi_t$ is assumed to be known. More details are in Appendix~\ref{appdx:det_exp}.

We compare the MSEs of six estimators, the FA3IPW, SFA3IPW, DM, AdaIPW, A2IPW, and FA2daIPW estimators. For the FA3IPW, SFA3IPW, and FA2daIPW estimators, we applied the two-step estimation. In each experiment, we have access to historical data with $T=1000$ and evaluation data with $N=1000$. With the SFA3IPW estimator, we discard the first $1000$ samples of the historical data. When estimating $f^*$ and $e^*$, we use the Nadaraya-Watson regression (NW) estimator, which is used in \citet{yang2002}. The resulting MSEs and their standard deviations (SDs) over $20$ replications of each experiment are shown in Tables~\ref{tbl:exp_table1} and \ref{tbl:exp_table2}. In the experimental results, when the RW policy is used, the FA3IPW and SFA3IPW estimators show preferable performances compared with the A2IPW estimator. This is because the performance of the A2IPW estimator is guaranteed only when the policy converges to a time-invariant policy. On the other hand, when the UCB policy is used, the A2IPW estimator also works well as we can expect. From the results, we recommend using the FA3IPW estimator when the policy does not converge. More importantly, in the case, we can construct confidence intervals from the FA3IPW estimator, but cannot construct it from the A2IPW estimator. In Appendix~\ref{appdx:det_exp}, we show the experimental results with other settings. 

\section{Conclusion}
This study presented solutions for causal inference from dependent samples obtained via bandit feedback. By utilizing evaluation data or sample splitting, we can standardize the variance of an MDS. Then, we can apply the CLT for an MDS to obtain the asymptotic normality of the proposed estimator. In experiments, the proposed estimators showed theoretically expected performances.

\section*{Broader Impact}
In various applications, such as A/B testing and clinical trials, the problem of causal inference from dependent samples arises. Because this study promotes the use of OPE for those applications by presenting an effective solution to this problem, the ethical issues related to such applications will become more critical. For example, in clinical trials, the proposed method also should take the criteria proposed by the US Food and Drug Administration \citep{fda} into account. 

Besides, although this study solved an essential problem with OPE, the limitations of methods for OPE should still be recognized. The existing methods for OPE try to find causality only from data without requiring domain-specific knowledge, but they sometimes lead to inappropriate decision making by producing incorrect results. When finding causality in the real world, it might be desirable to utilize domain-specific knowledge combined with the results obtained from machine learning methods. 

\bibliographystyle{icml2019}
\bibliography{arXiv.bbl}

\clearpage

\appendix

\section{Preliminaries}
\label{appdx;prelim}

\subsection{Mathematical Tools}

\begin{proposition}\label{prp:rules_for_ld}[Slutsky Theorem, \citet{greene2003econometric}, Theorem D.~16 1, p.~1117]  If $a_n\xrightarrow{d}a$ and $b_n\xrightarrow{p} b$, then
\begin{align*}
a_nb_n\xrightarrow{d} ba.
\end{align*}
\end{proposition}

\begin{definition}\label{dfn:uniint}[Uniformly Integrable, \citet{GVK126800421}, p.~191]  A sequence $\{A_t\}$ is said to be uniformly integrable if for every $\epsilon > 0$ there exists a number $c>0$ such that 
\begin{align*}
\mathbb{E}[|A_t|\cdot \mathbbm{1}[|A_t \geq c|]] < \epsilon
\end{align*}
for all $t$.
\end{definition}

\begin{proposition}\label{prp:suff_uniint}[Sufficient Conditions for Uniformly Integrable, \citet{GVK126800421}, Proposition~7.7, p.~191]  (a) Suppose there exist $r>1$ and $M<\infty$ such that $\mathbb{E}[|A_t|^r]<M$ for all $t$. Then $\{A_t\}$ is uniformly integrable. (b) Suppose there exist $r>1$ and $M < \infty$ such that $\mathbb{E}[|b_t|^r]<M$ for all $t$. If $A_t = \sum^\infty_{j=-\infty}h_jb_{t-j}$ with $\sum^\infty_{j=-\infty}|h_j|<\infty$, then $\{A_t\}$ is uniformly integrable.
\end{proposition}

\begin{proposition}[$L^r$ Convergence Theorem, \citet{loeve1977probability}]
\label{prp:lr_conv_theorem}
Let $0<r<\infty$, suppose that $\mathbb{E}\big[|a_n|^r\big] < \infty$ for all $n$ and that $a_n \xrightarrow{\mathrm{p}}a$ as $n\to \infty$. The following are equivalent: 
\begin{description}
\item{(i)} $a_n\to a$ in $L^r$ as $n\to\infty$;
\item{(ii)} $\mathbb{E}\big[|a_n|^r\big]\to \mathbb{E}\big[|a|^r\big] < \infty$ as $n\to\infty$; 
\item{(iii)} $\big\{|a_n|^r, n\geq 1\big\}$ is uniformly integrable.
\end{description}
\end{proposition}

\subsection{Martingale Limit Theorems}

\begin{proposition}
\label{prp:mrtgl_WLLN}[Weak Law of Large Numbers for Martingale, \citet{hall2014martingale}]
Let $\{S_n = \sum^{n}_{i=1} X_i, \Omega_{t}, t\geq 1\}$ be a martingale and $\{b_n\}$ a sequence of positive constants with $b_n\to\infty$ as $n\to\infty$. Then, writing $X_{ni} = X_i\mathbbm{1}[|X_i|\leq b_n]$, $1\leq i \leq n$, we have that $b^{-1}_n S_n \xrightarrow{\mathrm{p}} 0$ as $n\to \infty$ if 
\begin{description}
\item[(i)] $\sum^n_{i=1}P(|X_i| > b_n)\to 0$;
\item[(ii)] $b^{-1}_n\sum^n_{i=1}\mathbb{E}[X_{ni}\mid \Omega_{t-1}] \xrightarrow{\mathrm{p}} 0$, and;
\item[(iii)] $b^2_n \sum^n_{i=1}\big\{\mathbb{E}[X^2_{ni}] - \mathbb{E}\big[\mathbb{E}\big[X_{ni}\mid \Omega_{t-1}\big]\big]^2\big\}\to 0$.
\end{description}
\end{proposition}
\begin{remark} The weak law of large numbers for martingale holds when the random variable is bounded by a constant.
\end{remark}

\section{Proof of Theorem~\ref{thm:consist_var}}
\label{appdx:consist_var}
\begin{proof}
The variance $\sigma^{2*}_t$ can be calculated as follows:
\begin{align*}
&\sigma^{*2}_t=\mathbb{E}\left[\left(\sum^{K}_{a=1}\left\{\frac{\epol(a\mid X_t)\mathbbm{1}[A_t=a]\big\{Y_t - f^*(a, X_t)\big\}}{\pi_{t}(a\mid X_t, \Omega_{t - 1})} + \epol(a\mid X_t)f^*(a, X_t)\right\}-\theta_0\right)^2\right]\\
&=\mathbb{E}\left[\left(\sum^{K}_{a=1}\left\{\frac{\epol(a\mid X_t)\mathbbm{1}[A_t=a]\big\{Y_t(a) - f^*(a, X_t)\big\}}{\pi_{t}(a\mid X_t, \Omega_{t - 1})} + \epol(a\mid X_t)f^*(a, X_t)\right\}-\theta_0\right)^2\right]\\
&=\mathbb{E}\left[\sum^{K}_{a=1}\left\{\frac{\big(\epol(a\mid X_t)\big)^2\Big(e^*(a, X_t) - \big(f^*(a, X_t)\big)^2\Big)}{\pi_{t}(a\mid X_t, \Omega_{t - 1})} + \Big(\epol(a\mid X_t)f^*(a, X_t)-\theta_0\Big)^2\right\}\right].
\end{align*}
Therefore, what we want to show is
\begin{align*}
&\frac{1}{N} \sum^N_{i = 1} \sum^{K}_{a=1}\left\{\frac{\big(\epol(a\mid X_i)\big)^2\big(\hat{e}_{t-1}(a\mid X_i) - \hat{f}^2_{t-1}(a\mid X_i)\big)}{\pi_{t}(a\mid X_i, \Omega_{t - 1})} + \Big(\epol(a\mid X_i)\hat{f}_{t-1}(a, X_i) - \hat{\theta}_{t-1} \Big)^2\right\}\\
&-\mathbb{E}\left[\sum^{K}_{a=1}\left\{\frac{\big(\epol(a\mid X_t)\big)^2\Big(e^*(a, X_t) - \big(f^*(a, X_t)\big)^2\Big)}{\pi_{t}(a\mid X_t, \Omega_{t - 1})} + \Big(\epol(a\mid X_t)f^*(a, X_t)-\theta_0\Big)^2\right\}\right]\\
&\xrightarrow{\mathrm{p}} 0.
\end{align*}
From the assumptions, $\hat{f}_{t-1}(a, x)-f^*(a, x)\xrightarrow{\mathrm{p}}0$, $\hat{e}_{t-1}(a, x)-e^*(a, x)\xrightarrow{\mathrm{p}}0$, and $\hat{\theta}_{t-1}-\theta_0\xrightarrow{\mathrm{p}}0$ as $t\to\infty$, 
\begin{align*}
& \frac{1}{N} \sum^N_{t = 1} \sum^{K}_{a=1}\left\{\frac{\big(\epol(a\mid X_i)\big)^2\big(\hat{e}_{t-1}(a\mid X_i) - \hat{f}^2_{t-1}(a\mid X_i)\big)}{\pi_{t}(a\mid X_i, \Omega_{t - 1})} + \Big(\epol(a\mid X_i)\hat{f}_{t-1}(a, X_i) - \hat{\theta}_{t-1} \Big)^2\right\}\\
&- \frac{1}{N} \sum^N_{i =1 } \sum^{K}_{a=1}\left\{\frac{\big(\epol(a\mid X_i)\big)^2\Big(e^*(a, X_i) - \big(f^*(a, X_i)\big)^2\Big)}{\pi_{t}(a\mid X_i, \Omega_{t - 1})} + \Big(\epol(a\mid X_i)f^*(a, X_i) - \theta_0 \Big)^2\right\}\\
&\xrightarrow{\mathrm{p}} 0\ \ \ \mathrm{as}\ t\to \infty.
\end{align*}
Here, from the weak law of large numbers, as $N\to\infty$,
\begin{align*}
&\frac{1}{N} \sum^N_{i = 1} \sum^{K}_{a=1}\left\{\frac{\big(\epol(a\mid X_i)\big)^2\Big(e^*(a, X_i) - \big(f^*(a, X_i)\big)^2\Big)}{\pi_{t}(a\mid X_i, \Omega_{t - 1})} + \Big(\epol(a\mid X_i)f^*(a, X_i) - \theta_0 \Big)^2\right\}\\
&=\frac{1}{N} \sum^N_{i = 1} \sum^{K}_{a=1}\left\{\frac{\big(\epol(a\mid X_i)\big)^2\nu^*(a, X_i)\Big)}{\pi_{t}(a\mid X_i, \Omega_{t - 1})} + \Big(\epol(a\mid X_i)f^*(a, X_i) - \theta_0 \Big)^2\right\}\\
&\xrightarrow{\mathrm{p}}  \mathbb{E}\left[\sum^{K}_{a=1}\left\{\frac{\big(\epol(a\mid X_i)\big)^2\nu^*(a, X_i)}{\pi_{t}(a\mid X_i, \Omega_{t - 1})} + \big(\epol(a\mid X_i)f^*(a, X_i)\right\} - \theta_0 \big)^2\right].
\end{align*}

Therefore, as $t\to \infty$ and $N\to\infty$,
\begin{align*}
&\frac{1}{N} \sum^N_{i = 1} \sum^{K}_{a=1}\left\{\frac{\big(\epol(a\mid X_i)\big)^2\big(\hat{e}_{t-1}(a\mid X_i) - \hat{f}^2_{t-1}(a\mid X_i)\big)}{\pi_{t}(a\mid X_i, \Omega_{t - 1})} + \Big(\epol(a\mid X_i)\hat{f}_{t-1}(a, X_i) - \hat{\theta}_{t-1} \Big)^2\right\}\\
&\ \ \ - \mathbb{E}\left[\sum^{K}_{a=1}\left\{\frac{\big(\epol(a\mid X_i)\big)^2\nu^*(a, X_i)}{\pi_{t}(a\mid X_i, \Omega_{t - 1})}\right\} + \big(\epol(a\mid X_i)f^*(a, X_i) - \theta_0 \big)^2\right]\xrightarrow{\mathrm{p}} 0.
\end{align*}
\end{proof}

\section{Proof of Theorem~\ref{thm:main}}
\label{appdx:proof_thm:main}
\begin{proof}
We have
\begin{align*}
\sqrt{T}\big(\hat{\theta}_{T} - \theta_0\big) = \left(\frac{1}{T}\sum^{T}_{t=1}\frac{1}{\sqrt{g_{t, N}}}\right)^{-1}\sqrt{T}\frac{1}{T}\sum^{T}_{t=1}\hat{z}_t,
\end{align*}
where note that
\begin{align*}
&\hat{z}_t= \frac{1}{\sqrt{g_{t, N}}}\left(\sum^{K}_{a=1}\left\{\frac{\epol(a\mid X_t)\mathbbm{1}[A_t=a]\big\{Y_t - \hat{f}_{t-1}(a, X_t)\big\}}{\pi_{t}(a\mid X_t, \Omega_{t - 1})} + \epol(a\mid x)\hat{f}_{t-1}(a, X_t)\right\} - \theta_0\right).
\end{align*}
If we show
\begin{align}
\label{target1}
\sqrt{T}\frac{1}{T}\sum^{T}_{t=1}\hat{z}_t \xrightarrow{\mathrm{d}}\mathcal{N}\big(0, 1\big)\ \ \ \mathrm{as}\ T\to\infty, 
\end{align}
then we can show that
\begin{align*}
&\left(\frac{1}{T}\sum^{T}_{t=1}\frac{1}{\sqrt{g_{t, N}}}\right)\sqrt{T}\big(\hat{\theta}_{T} - \theta_0\big) = \sqrt{T}\frac{1}{T}\sum^{T}_{t=1}\hat{z}_t \xrightarrow{\mathrm{d}}\mathcal{N}\left(0, 1\right).
\end{align*}

To show \eqref{target1}, we use the central limit theorem for an MDS (Proposition~\ref{prp:marclt} in Appendix~\ref{appdx;prelim}). Because $\big\{\hat{z}_t\big\}^T_{t=1}$ is an MDS, we check the following three conditions of Proposition~\ref{prp:marclt}:
\begin{description}
\item[(a)] $\mathbb{E}\big[\hat{z}^2_t\big] = \nu^2_t > 0$ with $\big(1/T\big) \sum^{T}_{t=1}\nu^2_t\to \nu^2 > 0$;
\item[(b)] $\mathbb{E}\big[|\hat{z}_t|^r\big] < \infty$ for some $r>2$;
\item[(c)] $\big(1/T\big)\sum^{T}_{t=1}\hat{z}^2_t\xrightarrow{\mathrm{p}} \nu^2$. 
\end{description}

\subsection*{Step~1: Check of Condition~(a)}
We have
\begin{align*}
&\mathbb{E}\big[\hat{z}^2_t\big]=\mathbb{E}\left[\left\{\frac{1}{\sqrt{g_{t, N}}}\left(\sum^{K}_{a=1}\left\{\frac{\epol(a\mid X_t)\mathbbm{1}[A_t=a]\big\{Y_t - \hat{f}_{t-1}(a, X_t)\big\}}{\pi_{t}(a\mid X_t, \Omega_{t - 1})} + \epol(a\mid x)\hat{f}_{t-1}(a, X_t)\right\} - \theta_0\right)\right\}^2\right]\\
&=\mathbb{E}\left[\left\{\frac{1}{\sqrt{g_{t, N}}}\left(\sum^{K}_{a=1}\left\{\frac{\epol(a\mid X_t)\mathbbm{1}[A_t=a]\big\{Y_t - \hat{f}_{t-1}(a, X_t)\big\}}{\pi_{t}(a\mid X_t, \Omega_{t - 1})} + \epol(a\mid x)\hat{f}_{t-1}(a, X_t)\right\} - \theta_0\right)\right\}^2\right]\\
&\ \ \ -\mathbb{E}\left[\left\{\frac{1}{\sqrt{\sigma^{*2}_t}}\left(\sum^{K}_{a=1}\left\{\frac{\epol(a\mid X_t)\mathbbm{1}[A_t=a]\big\{Y_t - f^*(a, X_t)\big\}}{\pi_{t}(a\mid X_t, \Omega_{t - 1})} + \epol(a\mid x)f^*(a, X_t)\right\} - \theta_0\right)\right\}^2\right]\\
&\ \ \ +\mathbb{E}\left[\left\{\frac{1}{\sqrt{\sigma^{*2}_t}}\left(\sum^{K}_{a=1}\left\{\frac{\epol(a\mid X_t)\mathbbm{1}[A_t=a]\big\{Y_t - f^*(a, X_t)\big\}}{\pi_{t}(a\mid X_t, \Omega_{t - 1})} + \epol(a\mid x)f^*(a, X_t)\right\} - \theta_0\right)\right\}^2\right].
\end{align*}
Because the last term is $1$, We have
\begin{align*}
&\mathbb{E}\big[\hat{z}^2_t\big] - 1\\
&=\mathbb{E}\left[\left\{\frac{1}{\sqrt{g_{t, N}}}\left(\sum^{K}_{a=1}\left\{\frac{\epol(a\mid X_t)\mathbbm{1}[A_t=a]\big\{Y_t - \hat{f}_{t-1}(a, X_t)\big\}}{\pi_{t}(a\mid X_t, \Omega_{t - 1})} + \epol(a\mid x)\hat{f}_{t-1}(a, X_t)\right\} - \theta_0\right)\right\}^2\right]\\
&\ \ \ -\mathbb{E}\left[\left\{\frac{1}{\sqrt{\sigma^{*2}_t}}\left(\sum^{K}_{a=1}\left\{\frac{\epol(a\mid X_t)\mathbbm{1}[A_t=a]\big\{Y_t - f^*(a, X_t)\big\}}{\pi_{t}(a\mid X_t, \Omega_{t - 1})} + \epol(a\mid x)f^*(a, X_t)\right\} - \theta_0\right)\right\}^2\right].
\end{align*}
First, we show that the first and second terms vanishes asymptotically. We have
\begin{align*}
&\mathbb{E}\left[\left\{\frac{1}{\sqrt{g_{t, N}}}\left(\sum^{K}_{a=1}\left\{\frac{\epol(a\mid X_t)\mathbbm{1}[A_t=a]\big\{Y_t - \hat{f}_{t-1}(a, X_t)\big\}}{\pi_{t}(a\mid X_t, \Omega_{t - 1})} + \epol(a\mid x)\hat{f}_{t-1}(a, X_t)\right\} - \theta_0\right)\right\}^2\right]\\
&\ \ \ -\mathbb{E}\left[\left\{\frac{1}{\sqrt{\sigma^{*2}_t}}\left(\sum^{K}_{a=1}\left\{\frac{\epol(a\mid X_t)\mathbbm{1}[A_t=a]\big\{Y_t - f^*(a, X_t)\big\}}{\pi_{t}(a\mid X_t, \Omega_{t - 1})} + \epol(a\mid x)f^*(a, X_t)\right\} - \theta_0\right)\right\}^2\right]\\
&\leq \mathbb{E}\Bigg[\Bigg|\left\{\frac{1}{\sqrt{g_{t, N}}}\left(\sum^{K}_{a=1}\left\{\frac{\epol(a\mid X_t)\mathbbm{1}[A_t=a]\big\{y - \hat{f}_{t-1}(a, X_t)\big\}}{\pi_{t}(a\mid X_t, \Omega_{t - 1})} + \epol(a\mid x)\hat{f}_{t-1}(a, X_t)\right\} - \theta_0\right)\right\}^2\\
&\ \ \ \ \ \ \ \ \ -\left\{\frac{1}{\sqrt{\sigma^{*2}_t}}\left(\sum^{K}_{a=1}\left\{\frac{\epol(a\mid X_t)\mathbbm{1}[A_t=a]\big\{Y_t - f^*(a, X_t)\big\}}{\pi_{t}(a\mid X_t, \Omega_{t - 1})} + \epol(a\mid x)f^*(a, X_t)\right\} - \theta_0\right)\right\}^2\Bigg|\Bigg].
\end{align*}
From the boundedness of each variable in $\hat{z}_t$, there exist absolute constants $\tilde C_1$ and $\tilde C_2$ such that
\begin{align*}
&\leq \mathbb{E}\left[\Bigg|\left\{\frac{1}{\sqrt{g_{t, N}}}\left(\sum^{K}_{a=1}\left\{\frac{\epol(a\mid X_t)\mathbbm{1}[A_t=a]\big\{y - \hat{f}_{t-1}(a, X_t)\big\}}{\pi_{t}(a\mid X_t, \Omega_{t - 1})} + \epol(a\mid x)\hat{f}_{t-1}(a, X_t)\right\} - \theta_0\right)\right\}^2\right]\\
&\ \ \ -\mathbb{E}\left[\left\{\frac{1}{\sqrt{\sigma^{*2}_t}}\left(\sum^{K}_{a=1}\left\{\frac{\epol(a\mid X_t)\mathbbm{1}[A_t=a]\big\{Y_t - f^*(a, X_t)\big\}}{\pi_{t}(a\mid X_t, \Omega_{t - 1})} + \epol(a\mid x)f^*(a, X_t)\right\} - \theta_0\right)\right\}^2\Bigg|\right]\\
&\leq \mathbb{E}\left[\left|\frac{1}{\sqrt{g_{t, N}}} - \frac{1}{\sqrt{\sigma^{*2}_t}}\right|\right]\tilde C_1 +  \mathbb{E}\left[\left|\hat{f}_{t-1}(a, X_t) - f^*(a, X_t)\right|\right]\tilde C_2.
\end{align*}
From the continuous mapping theorem and assumption that $g_{t, N} - \sigma^{*2}_t\xrightarrow{\mathrm{p}}0$ as $t\to\infty$ and $N\to\infty$, we have $1/\sqrt{g_{t, N}} - 1/\sqrt{\sigma^{*2}_t}\xrightarrow{\mathrm{p}}0$ as $t\to\infty$ and $N\to\infty$ from the continuous mapping theorem. Then, if $1/\sqrt{g_{t,N}}$ is uniformly integrable, we can prove $\mathbb{E}[|1/\sqrt{g_{t, N}} - 1/\sqrt{\sigma^{*2}_t}|] \to 0$ as $t\to\infty$ using $L^r$-convergence theorem (Proposition~\ref{prp:lr_conv_theorem} in Appendix~\ref{appdx;prelim}).

Because $\hat{f}_{t-1}(a, x)$ is independent from $X_t$, we have $\mathbb{E}[\hat{f}_{t-1}(a, X_t)\mid X_t=x] = \mathbb{E}[\hat{f}_{t-1}(a, x)]$. For $\mathbb{E}[\hat{f}_{t-1}(a, x)]$ and fixed $x$, from the point convergence of $\hat{f}_{t-1}(a, x)$, we have 
\begin{align*}
\mathbb{E}\big[\big|\hat{f}_{t-1}(a, X_t) - \mathbb{E}[f^*(a, X_t)]\big|\mid X_t = x\big] = \mathbb{E}\big[\big|\hat{f}_{t-1}(a, x) - \mathbb{E}[f^*(a, x)]\big|\big] \to 0.
\end{align*}
Then, from the Lebesgue's dominated convergence theorem, for $\mathbb{E}[\hat{f}_{t-1}(a, X_t)\mid X_t=x]$, which pointwisely converges to $\mathbb{E}[f^*(a, X_t)\mid X_t=x]$, we can show that 
\begin{align*}
&\mathbb{E}_{X_t}\Big[\Big|\mathbb{E}[\hat{f}_{t-1}(a, X_t)\mid X_t] -  \mathbb{E}[f^*(a, X_t)\mid X_t]\Big|\Big]\to 0.
\end{align*}

We can show that $1/\sqrt{g_{t, N}}$ and $\hat{f}_{t-1}$ are uniformly integrable from the boundedness of $1/\sqrt{g_{t, N}}$ and $\hat{f}_{t-1}$ (Proposition~\ref{prp:suff_uniint} in Appendix~\ref{appdx;prelim}). Then, as $t\to\infty$ and $T\to\infty$,
\begin{align*}
&\mathbb{E}\left[\left\{\frac{1}{\sqrt{g_{t, N}}}\left(\sum^{K}_{a=1}\left\{\frac{\epol(a\mid X_t)\mathbbm{1}[A_t=a]\big\{Y_t - \hat{f}_{t-1}(a, X_t)\big\}}{\pi_{t}(a\mid X_t, \Omega_{t - 1})} + \epol(a\mid x)\hat{f}_{t-1}(a, X_t)\right\} - \theta_0\right)\right\}^2\right]\\
&\ \ \ -\mathbb{E}\left[\left\{\frac{1}{\sqrt{\sigma^{*2}_t}}\left(\sum^{K}_{a=1}\left\{\frac{\epol(a\mid X_t)\mathbbm{1}[A_t=a]\big\{Y_t - f^*(a, X_t)\big\}}{\pi_{t}(a\mid X_t, \Omega_{t - 1})} + \epol(a\mid x)f^*(a, X_t)\right\} - \theta_0\right)\right\}^2\right]\\
&\leq \mathbb{E}\left[\left|\frac{1}{\sqrt{g_{t, N}}} - \frac{1}{\sqrt{\sigma^{*2}_t}}\right|\right]\tilde C_1 +  \mathbb{E}\big[\left|\hat{f}_{t-1}(a, X_t) - f^*(a, X_t)\right|\big]\tilde C_2\to 0.
\end{align*}
Therefore, for any $\epsilon > 0$, there exists $\tilde t > 0$ and $\tilde{T}>0$ such that, for $T > \max\left\{\tilde{T}, \tilde{t}/r\right\}$, 
\begin{align*}
&\mathbb{E}\big[\hat{z}^2_t\big] - 1\\
&=\frac{1}{T} \sum^{T}_{t=1}\Bigg(\mathbb{E}\left[\left\{\frac{1}{\sqrt{g_{t, N}}}\left(\sum^{K}_{a=1}\left\{\frac{\epol(a\mid X_t)\mathbbm{1}[A_t=a]\big\{Y_t - \hat{f}_{t-1}(a, X_t)\big\}}{\pi_{t}(a\mid X_t, \Omega_{t - 1})} + \epol(a\mid x)\hat{f}_{t-1}(a, X_t)\right\} - \theta_0\right)\right\}^2\right]\\
&\ \ \ \ \ \ \ \ \ \ \ \ \ \ \ \ \ \ \ \ \ \ \ \ \ \ \ \ -\mathbb{E}\left[\left\{\frac{1}{\sqrt{\sigma^{*2}_t}}\left(\sum^{K}_{a=1}\left\{\frac{\big(\epol(a\mid X_t)\big)^2\mathbbm{1}[A_t=a]\big\{Y_t - f^*(a, X_t)\big\}}{\pi_{t}(a\mid X_t, \Omega_{t - 1})} + \epol(a\mid x)f^*(a, X_t)\right\} - \theta_0\right)\right\}^2\right]\Bigg)\\
&\leq \tilde t/T + \epsilon.
\end{align*}
Thus, we have $\big(1/T\big) \sum^{T}_{t=1}\nu^2_t- 1 \leq \tilde t/T + \epsilon \to 0$ as $T\to\infty$.

\subsubsection*{Step~2: Check of Condition~(b)}
We can prove this condition using the assumption of the boundedness of the random variables.

\subsubsection*{Step~3: Check of Condition~(c)} 
To show $(1/T)\sum^{T}_{t=1}\hat{z}^2_t \xrightarrow{\mathrm{p}} \nu^2$, we show (i) $(1/T)\sum^{T}_{t=1}\hat{z}^2_t - (1/T)\sum^{T}_{t=1}\mathbb{E}[\hat{z}^2_t\mid \Omega_{t-1}]\xrightarrow{\mathrm{p}}0$; (ii) $(1/T)\sum^{T}_{t=1}\mathbb{E}[\hat{z}^2_t\mid \Omega_{t-1}] - \nu^2\xrightarrow{\mathrm{p}}0$, where $\nu^2 = \mathbb{E}\left[\left\{\frac{1}{\sqrt{\sigma^{*2}_t}}\left(\sum^{K}_{a=1}\left\{\frac{\big(\epol(a\mid X_t)\big)^2\mathbbm{1}[A_t=a]\big\{Y_t - f^*(a, X_t)\big\}}{\pi_{t}(a\mid X_t, \Omega_{t - 1})} + \epol(a\mid x)f^*(a, X_t)\right\} - \theta_0\right)\right\}^2\right]=1$.

\paragraph{Proof of (i):} We can prove (i) directly applying law of large numbers for an MDS to an MDS $\phi_t =\hat{z}^2_t - \mathbb{E}[\hat{z}^2_t\mid \Omega_{t-1}]$. The conditions for using the law of large numbers are satisfied by the boundedness of components of $\hat{z}_t$.

\paragraph{Proof of (ii)}
Next, we show that
\begin{align*}
\frac{1}{T}\sum^{T}_{t=1} \mathbb{E}\big[\hat{z}^2_t\mid \Omega_{t-1}\big] - \nu^2\xrightarrow{\mathrm{p}} 0.
\end{align*}
From Markov's inequality, for $\varepsilon > 0$, we have
\begin{align*}
&\mathbb{P}\left(\left|\frac{1}{T}\sum^{T}_{t=1}\mathbb{E}\big[\hat{z}^2_t\mid \Omega_{t-1}\big] - \nu^2\right| \geq \varepsilon\right)\leq \frac{\mathbb{E}\left[\left|\frac{1}{T}\sum^{T}_{t=1}\mathbb{E}\big[\hat{z}^2_t\mid \Omega_{t-1}\big] - \nu^2\right|\right]}{\varepsilon}\leq \frac{\frac{1}{T}\sum^{T}_{t=1}\mathbb{E}\left[\big|\mathbb{E}\big[\hat{z}^2_t\mid \Omega_{t-1}\big] - \nu^2\big|\right]}{\varepsilon}.
\end{align*}
Then, we consider showing $\mathbb{E}\left[\big|\mathbb{E}\big[z^2_t\mid \Omega_{t-1}\big] - \nu^2\big|\right] \to 0$. Here, we have
\begin{align*}
&\mathbb{E}\left[\big|\mathbb{E}\big[\hat{z}^2_t\mid \Omega_{t-1}\big] - \nu^2\big|\right]\\
&=\mathbb{E}\Bigg[\Bigg|\mathbb{E}\Bigg[\Bigg\{\frac{1}{\sqrt{g_{t, N}}}\left(\sum^{K}_{a=1}\left\{\frac{\epol(a\mid X_t)\mathbbm{1}[A_t=a]\big\{Y_t - \hat{f}_{t-1}(a, X_t)\big\}}{\pi_{t}(a\mid X_t, \Omega_{t - 1})} + \epol(a\mid x)\hat{f}_{t-1}(a, X_t)\right\} - \theta_0\right)\Bigg\}^2 \\
&\ \ \ - \Bigg\{\frac{1}{\sqrt{\sigma^{*2}_t}}\left(\sum^{K}_{a=1}\left\{\frac{\epol(a\mid X_t)\mathbbm{1}[A_t=a]\big\{Y_t - f^*(a, X_t)\big\}}{\pi_{t}(a\mid X_t, \Omega_{t - 1})} + \epol(a\mid x)f^*(a, X_t)\right\} - \theta_0\right)\Bigg\}^2 \mid \Omega_{t-1}\Bigg]\Bigg|\Bigg]\\
&=\mathbb{E}\Bigg[\Bigg|\mathbb{E}\Bigg[\mathbb{E}\Bigg[\Bigg\{\frac{1}{\sqrt{g_{t, N}}}\left(\sum^{K}_{a=1}\left\{\frac{\epol(a\mid X_t)\mathbbm{1}[A_t=a]\big\{Y_t - \hat{f}_{t-1}(a, X_t)\big\}}{\pi_{t}(a\mid X_t, \Omega_{t - 1})} + \epol(a\mid x)\hat{f}_{t-1}(a, X_t)\right\} - \theta_0\right)\Bigg\}^2 \\
&\ \ \ -  \Bigg\{\frac{1}{\sqrt{\sigma^{*2}_t}}\left(\sum^{K}_{a=1}\left\{\frac{\epol(a\mid X_t)\mathbbm{1}[A_t=a]\big\{Y_t - f^*(a, X_t)\big\}}{\pi_{t}(a\mid X_t, \Omega_{t - 1})} + \epol(a\mid x)f^*(a, X_t)\right\} - \theta_0\right)\Bigg\}^2 \mid X_t, \Omega_{t-1}\Bigg] \mid \Omega_{t-1}\Bigg]\Bigg|\Bigg].
\end{align*}
Then, by using Jensen's inequality, 
\begin{align*}
&\mathbb{E}\left[\big|\mathbb{E}\big[z^2_t\mid \Omega_{t-1}\big] - \nu^2\big|\right]\\
&\leq \mathbb{E}\Bigg[\mathbb{E}\Bigg[\Bigg|\mathbb{E}\Bigg[\Bigg\{\frac{1}{\sqrt{g_{t, N}}}\left(\sum^{K}_{a=1}\left\{\frac{\epol(a\mid X_t)\mathbbm{1}[A_t=a]\big\{Y_t - \hat{f}_{t-1}(a, X_t)\big\}}{\pi_{t}(a\mid X_t, \Omega_{t - 1})} + \epol(a\mid x)\hat{f}_{t-1}(a, X_t)\right\} - \theta_0\right)\Bigg\}^2\\
&\ \ \ - \Bigg\{ \frac{1}{\sqrt{\sigma^{*2}_t}}\left(\sum^{K}_{a=1}\left\{\frac{\epol(a\mid X_t)\mathbbm{1}[A_t=a]\big\{Y_t - f^*(a, X_t)\big\}}{\pi_{t}(a\mid X_t, \Omega_{t - 1})} + \epol(a\mid x)f^*(a, X_t)\right\} - \theta_0\right) \Bigg\}^2\Bigg| \mid \Omega_{t-1}\Bigg]\Bigg]\\
&= \mathbb{E}\Bigg[\Bigg|\mathbb{E}\Bigg[\Bigg\{\frac{1}{\sqrt{g_{t, N}}}\left(\sum^{K}_{a=1}\left\{\frac{\epol(a\mid X_t)\mathbbm{1}[A_t=a]\big\{Y_t - \hat{f}_{t-1}(a, X_t)\big\}}{\pi_{t}(a\mid X_t, \Omega_{t - 1})} + \epol(a\mid x)\hat{f}_{t-1}(a, X_t)\right\} - \theta_0\right)\Bigg\}^2\\
&\ \ \ -  \Bigg\{\frac{1}{\sqrt{\sigma^{*2}_t}}\left(\sum^{K}_{a=1}\left\{\frac{\epol(a\mid X_t)\mathbbm{1}[A_t=a]\big\{Y_t - f^*(a, X_t)\big\}}{\pi_{t}(a\mid X_t, \Omega_{t - 1})} + \epol(a\mid x)f^*(a, X_t)\right\} - \theta_0\right) \Bigg\}^2\mid X_t, \Omega_{t-1}\Bigg] \Bigg| \Bigg].
\end{align*}
Because $\hat f_{t-1}$ and $g_{t, N}$ are constructed from $\Omega_{t-1}$,
\begin{align*}
&\mathbb{E}\left[\big|\mathbb{E}\big[z^2_t\mid \Omega_{t-1}\big] - \nu^2\big|\right]\\
&\leq \mathbb{E}\Bigg[\Bigg|\mathbb{E}\Bigg[\Bigg\{\frac{1}{\sqrt{g_{t, N}}}\left(\sum^{K}_{a=1}\left\{\frac{\epol(a\mid X_t)\mathbbm{1}[A_t=a]\big\{Y_t - \hat{f}_{t-1}(a, X_t)\big\}}{\pi_{t}(a\mid X_t, \Omega_{t - 1})} + \epol(a\mid x)\hat{f}_{t-1}(a, X_t)\right\} - \theta_0\right)\Bigg\}^2\\
&\ \ \ -  \Bigg\{\frac{1}{\sqrt{\sigma^{*2}_t}}\left(\sum^{K}_{a=1}\left\{\frac{\epol(a\mid X_t)\mathbbm{1}[A_t=a]\big\{Y_t - f^*(a, X_t)\big\}}{\pi_{t}(a\mid X_t, \Omega_{t - 1})} + \epol(a\mid x)f^*(a, X_t)\right\} - \theta_0\right)\Bigg\}^2 \mid X_t, \hat{f}_{t-1}, g_{t, N}\Bigg] \Bigg| \Bigg].
\end{align*}
From the results of Step~1, there exist $\tilde{C}_1$ and $\tilde{C}_2$ such that
\begin{align*}
&\mathbb{E}\left[\big|\mathbb{E}\big[z^2_t\mid \Omega_{t-1}\big] - \nu^2\big|\right]\\
&\leq \mathbb{E}\Bigg[\Bigg|\mathbb{E}\Bigg[\Bigg\{\frac{1}{\sqrt{g_{t, N}}}\left(\sum^{K}_{a=1}\left\{\frac{\epol(a\mid X_t)\mathbbm{1}[A_t=a]\big\{Y_t - \hat{f}_{t-1}(a, X_t)\big\}}{\pi_{t}(a\mid X_t, \Omega_{t - 1})} + \epol(a\mid x)\hat{f}_{t-1}(a, X_t)\right\} - \theta_0\right)\Bigg\}^2\\
&\ \ \ -  \Bigg\{\frac{1}{\sqrt{\sigma^{*2}_t}}\left(\sum^{K}_{a=1}\left\{\frac{\epol(a\mid X_t)\mathbbm{1}[A_t=a]\big\{Y_t - f^*(a, X_t)\big\}}{\pi_{t}(a\mid X_t, \Omega_{t - 1})} + \epol(a\mid x)f^*(a, X_t)\right\} - \theta_0\right)\Bigg\}^2 \mid X_t, \hat{f}_{t-1}, g_{t, N}\Bigg] \Bigg| \Bigg]\\
&\leq \mathbb{E}\left[\left|\frac{1}{\sqrt{g_{t, N}}} - \frac{1}{\sqrt{\sigma^{*2}_t}}\right|\right]\tilde C_1 +  \mathbb{E}\big[\left|\hat{f}_{t-1}(a, X_t) - f^*(a, X_t)\right|\big]\tilde C_2.
\end{align*}

Then, from $L^r$ convergence theorem, by using point convergence and the boundedness of $g_{t, N}$ and $\hat{f}_{t-1}$, we have $\mathbb{E}\left[\big|\mathbb{E}\big[z^2_t\mid \Omega_{t-1}\big] - \nu^2\big|\right]\to 0$ as $t\to\infty$ and $N\to\infty$. Therefore, as $T\to\infty$ and $N\to\infty$, 
\begin{align*}
&\mathbb{P}\left(\left|\frac{1}{T}\sum^T_{t=1}\mathbb{E}\big[\hat{z}^2_t\mid \Omega_{t-1}\big] - \nu^2\right| \geq \varepsilon\right) \leq \frac{\frac{1}{T}\sum^T_{t=1}\mathbb{E}\left[\big|\mathbb{E}\big[\hat{z}^2_t\mid \Omega_{t-1}\big] - \nu^2\big|\right]}{\varepsilon} \to 0.
\end{align*}
As a conclusion, 
\begin{align*}
&\frac{1}{T}\sum^T_{t=1}\hat{z}^2_t - \nu^2 = \frac{1}{T}\sum^T_{t=1}\big(\hat{z}^2_t - \mathbb{E}\left[z^2_t\mid \Omega_{t-1}\big] + \mathbb{E}\big[z^2_t\mid \Omega_{t-1}\big] - \nu^2\right)\xrightarrow{\mathrm{p}} 0.
\end{align*}

\subsubsection*{Conclusion}
From the results of Steps~1--3, we showed the theorem. 
\end{proof}

\begin{remark}[Donsker PropeTies]
Readers might feel that functions need to satisfy some assumptions such as Donsker propeTies \citep{ChernozhukovVictor2018Dmlf}. 
\end{remark}

\section{Details of Sample Splitting}
\label{appdx:sample_splitting}
As well as Theorem~\ref{thm:consist_var}, we prove the consistency of $g_{t, (1-r)T}$ in the following theorem.
\begin{corollary}
\label{cor:consist_var}
Suppose that, for all $x\in\mathcal{X}$ and $k\in\mathcal{A}$, $\hat{f}_{t-1}(a, x)-f^*(a, x)\xrightarrow{\mathrm{p}}0$, $\hat{e}_{t-1}(a, x)-e^*(a, x)\xrightarrow{\mathrm{p}}0$ and $\hat{\theta}_t\xrightarrow{\mathrm{p}}\theta_0$ as $t\to\infty$. Then, for $\pi_t\in\Pi$, we have $g'_{t, (1-r)T} - \sigma^{*2}_t\xrightarrow{\mathrm{p}}0$ as $t\to\infty$ and $T\to\infty$, where $\sigma^{*2}_t = \mathbb{E}\left[\sum^{K}_{a=1}\left\{\frac{\big(\epol(a\mid X_t)\big)^2\nu^*(a, X_t)}{\pi_{t}(a\mid X_t, \Omega_{t - 1})} + \Big(\epol(a\mid X_t)f^*(a, X_t)-\theta_0\Big)^2\right\}\right]$.
\end{corollary}
Then, using $g_{t, (1-r)T}$, we define a FA3IPW estimator with sample splitting $\hat{\theta}^{\mathrm{FA3IPW\mathchar`-SS}}_{rT}$ as $\left(\sum^{rT}_{t=1}\frac{1}{\sqrt{g_{t, (1-r)T}}}\right)^{-1}\sum^{rT}_{t=1}\frac{1}{\sqrt{g_{t, (1-r)T}}}\left\{\frac{\epol(a\mid X_t)\mathbbm{1}[A_t=a]\big\{Y_t - \hat{f}_{t-1}(a, X_t)\big\}}{\pi_{t}(a\mid X_t, \Omega_{t - 1})} + \epol(a\mid x)\hat{f}_{t-1}(a, X_t)\right\}$. The asymptotic distribution can be obtained from Theorem~\ref{thm:asymptotic_distribution_a3ipw} as follows. 
\begin{corollary}[Asymptotic Distribution of the Proposed Estimator]
Suppose that
\begin{description}
\item[(i)] there exist $C_2$ and $C_3$ such that $|\hat{f}_t|\leq C_2$ and $0 < 1/\sqrt{g_{t, (1-r)T}} \leq C_3$;
\item[(ii)] Pointwise convergence in probability of $\hat{f}_t$, i.e., $\forall x \in \mathcal{X}$, $\hat{f}_t(a, x)- f^*(a, x)\xrightarrow{\mathrm{p}}0$ as $t\to \infty$;
\item[(ii)] $g_{t, (1-r)T}- \sigma^{*2}_t\xrightarrow{\mathrm{p}}0$ as $t\to \infty$ and $T\to \infty$.
\end{description}
Then, under Assumptions~\ref{asm:DGP}, $\left(\frac{1}{\sqrt{rT}}\sum^{rT}_{t=1}\frac{1}{\sqrt{g_{t, (1-r)T}}}\right)\left(\hat{\theta}^{\mathrm{FA3IPW}}_{rT} - \theta_0\right) \xrightarrow{\mathrm{d}} \mathcal{N}\left(0, 1\right)$ as $T\to \infty$.
\end{corollary}
Thus, we can obtain an estimator with asymptotic normality without assuming the access to $\mathcal{E}_N$.

\begin{table}[t]
\label{Dataset}
\caption{Specification of datasets}
\begin{center}
\scalebox{0.9}[0.9]{\begin{tabular}{cccc}
\hline
Dataset&the number of samples &Dimension &the number of classes\\
\hline
 satimage & 4,435  &  35 & 6 \\
 pendigits & 7,496 &  16 & 10\\
 mnist & 60,000 &  780 & 10 \\
 dna& 2,000  &  180 & 3 \\
 letter      & 15,000 &  16 & 26 \\
 sensorless& 58,509 &  48 & 11 \\
 connect-4& 67,557 &  126 & 3\\
 covtype& 581,012 &  54 & 7 \\
\end{tabular}}
\end{center}
\end{table}

\section{Details of Experiments}
\label{appdx:det_exp}
In this section, we show the details and additional results with different settings of Section~\ref{sec:exp}. First, we show the description of the datasets in Table. All datasets are downloaded from \url{https://www.csie.ntu.edu.tw/~cjlin/libsvmtools/datasets/}. For these datasets, in the experiment shown in Section~\ref{sec:exp}, we use a policy $\pi_t = 0.7\pi^a_t+0.3\pi^u_t$ with $2,000$ samples. We use the UCB and RW policy for the adaptive policy. The UCB algorithm follows the LinUCB proposed by \citet{Wei2011}. In the random walk policy, we update the policy as follows:
\begin{description}
\item[(i)] At the first period, the policy $\pi^a_1$ is given from the uniform distribution with a normalization for becoming $\sum^K_{k=1}\pi^a_1(k\mid X_1) = 1$.
\item[(ii)] At a period $t$, we update the policy as $\pi^a_{t+1}(k\mid X_{t+1}) = \pi^a_t(k\mid X_t) + \eta_{k,t}$, where $\eta_{k,t}$ is a random variable drawn from the normal distribution.
\item[(iii)] For satisfying the definition of the probability, we normalize $\pi^a_{t+1}(k\mid X_{t+1})$ to become $\sum^K_{k=1}\pi^a_{t+1}(k\mid X_{t+1}) = 1$.
\item[(iv)] Iterate (ii) and (iii) until the trial ends.
\end{description}
For simplicity, we do not use context in the RW policy. For nonparametric regression, we only use the Nadaraya-Watson regression in Section~\ref{sec:exp}. The evaluation policy is fixed at $\epol = 0.7\pi^d+0.3\pi^u$. The policy $\pi_d$ is a prediction of the logistic regression trained by all samples in $\mathcal{S}_T$. If the logistic regression predicts a class $a\in\mathcal{A}$ as the true label given $X_t$, the policy becomes $\pi_d(a, X_t) = 1$ and $\pi_d(a', X_t) = 0$ for $a'\neq a$.

In Table~\ref{tbl:appdx:exp_table1},  we use a policy $\pi_t = 0.7\pi^a_t+0.3\pi^u_t$ with $T=1,000$ and $N=1,000$. For nonparametric regression, we use the K-nearest neighbor regression.  

In Table~\ref{tbl:appdx:exp_table2},  we use a policy $\pi_t = 0.7\pi^a_t+0.3\pi^u_t$ with $T=500$ and $N=1,000$. For nonparametric regression, we use the Nadaraya-Watson regression.

In Table~\ref{tbl:appdx:exp_table3},  we use a policy $\pi_t = 0.7\pi^a_t+0.3\pi^u_t$ with $T=500$ and $N=1,000$. For nonparametric regression, we use the K-nearest neighbor regression.  

In Table~\ref{tbl:appdx:exp_table4},  we use a policy $\pi_t = 0.5\pi^a_t+0.5\pi^u_t$ with $T=1,000$ and $N=1,000$. For nonparametric regression, we use the Nadaraya-Watson regression.  

In Table~\ref{tbl:appdx:exp_table5},  we use a policy $\pi_t = 0.5\pi^a_t+0.5\pi^u_t$ with $T=1,000$ and $N=1,000$. For nonparametric regression, we use the K-nearest neighbor regression.  

The proposed FA3IPW estimator performs well in many datasets. Especially when we use the RW policy, the performance of the estimator tends to be better than the other methods.

\begin{table*}[t]
\caption{Experimental results under the RW policy (the upper table) and UCB policy (the lower table) under a policy $\pi_t = 0.7\pi^a_t+0.3\pi^u_t$ with $T=1,000$ and $N=1,000$, and the K-nearest neighbor regression. The method with the lowest MSE is in bold.} 
\begin{center}
\medskip
\label{tbl:appdx:exp_table1}
\scalebox{0.65}[0.65]{
\begin{tabular}{l|rr|rr|rr|rr|rr|rr|rr|rr}
\toprule
Datasets &  \multicolumn{2}{c|}{satimage}& \multicolumn{2}{c|}{pendigits}& \multicolumn{2}{c|}{mnist}& \multicolumn{2}{c|}{dna}& \multicolumn{2}{c|}{letter}& \multicolumn{2}{c|}{sensorless}& \multicolumn{2}{c|}{connect-4}& \multicolumn{2}{c}{covtype} \\
Metrics &      MSE &      SD &      MSE &      SD &      MSE &      SD &      MSE &      SD &      MSE &      SD &     MSE &     SD &      MSE &     SD &     MSE &     SD \\
\hline
FA3IPW &  0.024 &  0.155 &  0.049 &  0.217 &  0.078 &  0.280 &  {\bf 0.017} &  0.131 &  0.187 &  0.422 &  0.071 &  0.261 &  {\bf 0.022} &  0.132 &  {\bf 0.044} &  0.208 \\
SFA3IPW &  {\bf 0.023} &  0.148 &  0.058 &  0.239 &  {\bf 0.075} &  0.274 &  0.021 &  0.144 &  {\bf 0.170} &  0.408 &  {\bf 0.067} &  0.259 &  0.023 &  0.140 &  0.050 &  0.221 \\
DM &  0.051 &  0.147 &  0.212 &  0.192 &  0.460 &  0.190 &  0.368 &  0.147 &  0.437 &  0.154 &  0.271 &  0.187 &  0.123 &  0.151 &  0.073 &  0.188 \\
AdaIPW &  0.042 &  0.203 &  0.075 &  0.267 &  0.099 &  0.315 &  0.042 &  0.204 &  0.189 &  0.423 &  0.117 &  0.341 &  0.073 &  0.205 &  0.058 &  0.242 \\
A2IPW &  0.024 &  0.155 &  {\bf 0.048} &  0.216 &  0.078 &  0.279 &  0.018 &  0.135 &  0.173 &  0.408 &  0.074 &  0.270 &  0.025 &  0.133 &  0.047 &  0.216 \\
FA2daIPW &  0.042 &  0.204 &  0.074 &  0.265 &  0.099 &  0.314 &  0.039 &  0.196 &  0.203 &  0.434 &  0.109 &  0.328 &  0.060 &  0.205 &  0.055 &  0.234 \\
\bottomrule
\end{tabular}
} 
\end{center}

\begin{center}
\scalebox{0.65}[0.65]{
\begin{tabular}{l|rr|rr|rr|rr|rr|rr|rr|rr}
\toprule
Datasets &  \multicolumn{2}{c|}{satimage}& \multicolumn{2}{c|}{pendigits}& \multicolumn{2}{c|}{mnist}& \multicolumn{2}{c|}{dna}& \multicolumn{2}{c|}{letter}& \multicolumn{2}{c|}{sensorless}& \multicolumn{2}{c|}{connect-4}& \multicolumn{2}{c}{covtype} \\
Metrics &      MSE &      SD &      MSE &      SD &      MSE &      SD &      MSE &      SD &      MSE &      SD &     MSE &     SD &      MSE &     SD &     MSE &     SD \\
\hline
FA3IPW &  0.061 &  0.246 &  0.100 &  0.316 &  0.130 &  0.359 &  0.036 &  0.183 &  0.223 &  0.471 &  0.077 &  0.266 &   4.627 &  1.969 &  0.049 &  0.222 \\
SFA3IPW &  0.062 &  0.246 &  0.126 &  0.344 &  0.156 &  0.393 &  {\bf 0.032} &  0.169 &  {\bf 0.164} &  0.405 &  0.092 &  0.293 &   4.150 &  1.918 &  0.052 &  0.226 \\
DM &  {\bf 0.014} &  0.117 &  {\bf 0.020} &  0.127 &  0.401 &  0.158 &  0.358 &  0.164 &  0.452 &  0.162 &  0.256 &  0.195 &  {\bf 0.135} &  0.108 &  0.075 &  0.183 \\
AdaIPW &  0.089 &  0.296 &  0.129 &  0.350 &  0.197 &  0.438 &  0.080 &  0.272 &  0.312 &  0.556 &  0.106 &  0.309 &  12.278 &  2.869 &  0.086 &  0.292 \\
A2IPW &  0.067 &  0.259 &  0.098 &  0.313 &  0.128 &  0.355 &  0.037 &  0.184 &  0.193 &  0.439 &  {\bf 0.076} &  0.266 &   4.100 &  1.882 &  {\bf 0.042} &  0.206 \\
FA2daIPW &  0.084 &  0.287 &  0.125 &  0.348 &  0.207 &  0.452 &  0.079 &  0.271 &  0.381 &  0.614 &  0.101 &  0.300 &  10.234 &  2.581 &  0.090 &  0.296 \\
\bottomrule
\end{tabular}
} 
\end{center}

\end{table*}

\begin{table*}[t]
\caption{Experimental results under the RW policy (upper table) and UCB policy (lower table) under a policy $\pi_t = 0.7\pi^a_t+0.3\pi^u_t$ with $T=500$ and $N=1,000$, and the Nadaraya-Watson regression. The method with the lowest MSE is in bold.} 
\begin{center}
\medskip
\label{tbl:appdx:exp_table2}
\scalebox{0.65}[0.65]{
\begin{tabular}{l|rr|rr|rr|rr|rr|rr|rr|rr}
\toprule
Datasets &  \multicolumn{2}{c|}{satimage}& \multicolumn{2}{c|}{pendigits}& \multicolumn{2}{c|}{mnist}& \multicolumn{2}{c|}{dna}& \multicolumn{2}{c|}{letter}& \multicolumn{2}{c|}{sensorless}& \multicolumn{2}{c|}{connect-4}& \multicolumn{2}{c}{covtype} \\
Metrics &      MSE &      SD &      MSE &      SD &      MSE &      SD &      MSE &      SD &      MSE &      SD &     MSE &     SD &      MSE &     SD &     MSE &     SD \\
\hline
FA3IPW &  0.052 &  0.215 &  0.116 &  0.338 &  {\bf 0.094} &  0.296 &  {\bf 0.040} &  0.198 &  0.195 &  0.441 &  0.097 &  0.291 &  {\bf 0.031} &  0.173 &  0.072 &  0.267 \\
SFA3IPW &  {\bf 0.045} &  0.202 &  {\bf 0.099} &  0.314 &  0.134 &  0.353 &  0.048 &  0.220 &  0.267 &  0.516 &  0.127 &  0.347 &  0.041 &  0.201 &  {\bf 0.067} &  0.254 \\
DM &  0.096 &  0.148 &  0.466 &  0.207 &  0.404 &  0.231 &  0.245 &  0.197 &  0.469 &  0.162 &  0.384 &  0.178 &  0.117 &  0.206 &  0.127 &  0.217 \\
AdaIPW &  0.082 &  0.283 &  0.144 &  0.375 &  0.121 &  0.336 &  0.070 &  0.264 &  0.196 &  0.442 &  0.114 &  0.322 &  0.077 &  0.221 &  0.121 &  0.346 \\
A2IPW &  0.051 &  0.213 &  0.113 &  0.333 &  0.095 &  0.293 &  0.044 &  0.207 &  {\bf 0.190} &  0.436 &  {\bf 0.096} &  0.295 &  {\bf 0.031} &  0.169 &  0.075 &  0.273 \\
FA2daIPW &  0.089 &  0.295 &  0.149 &  0.384 &  0.117 &  0.338 &  0.061 &  0.246 &  0.202 &  0.449 &  0.111 &  0.314 &  0.065 &  0.225 &  0.116 &  0.340 \\
\bottomrule
\end{tabular}
} 
\end{center}

\begin{center}
\scalebox{0.65}[0.65]{
\begin{tabular}{l|rr|rr|rr|rr|rr|rr|rr|rr}
\toprule
Datasets &  \multicolumn{2}{c|}{satimage}& \multicolumn{2}{c|}{pendigits}& \multicolumn{2}{c|}{mnist}& \multicolumn{2}{c|}{dna}& \multicolumn{2}{c|}{letter}& \multicolumn{2}{c|}{sensorless}& \multicolumn{2}{c|}{connect-4}& \multicolumn{2}{c}{covtype} \\
Metrics &      MSE &      SD &      MSE &      SD &      MSE &      SD &      MSE &      SD &      MSE &      SD &     MSE &     SD &      MSE &     SD &     MSE &     SD \\
\hline
FA3IPW &  0.057 &  0.239 &  0.248 &  0.497 &  0.160 &  0.396 &  0.048 &  0.216 &  {\bf 0.166} &  0.407 &  0.160 &  0.400 &  2.245 &  1.291 &  {\bf 0.059} &  0.237 \\
SFA3IPW &  0.054 &  0.232 &  0.295 &  0.542 &  0.257 &  0.495 &  0.051 &  0.224 &  0.243 &  0.492 &  0.183 &  0.423 &  2.657 &  1.487 &  0.077 &  0.271 \\
DM &  {\bf 0.036} &  0.139 &  0.303 &  0.128 &  0.336 &  0.204 &  0.238 &  0.200 &  0.465 &  0.144 &  0.391 &  0.198 &  {\bf 0.122} &  0.191 &  0.111 &  0.216 \\
AdaIPW &  0.161 &  0.401 &  0.416 &  0.642 &  0.303 &  0.548 &  0.079 &  0.276 &  0.204 &  0.451 &  0.175 &  0.417 &  9.985 &  2.344 &  0.099 &  0.310 \\
A2IPW &  0.064 &  0.254 &  {\bf 0.220} &  0.469 &  {\bf 0.149} &  0.384 &  {\bf 0.046} &  0.213 &  0.171 &  0.413 &  {\bf 0.149} &  0.386 &  2.102 &  1.251 &  0.065 &  0.247 \\
FA2daIPW &  0.187 &  0.429 &  0.442 &  0.662 &  0.276 &  0.520 &  0.079 &  0.275 &  0.206 &  0.453 &  0.195 &  0.440 &  8.764 &  2.242 &  0.096 &  0.305 \\
\bottomrule
\end{tabular}
} 
\end{center}

\end{table*}

\begin{table*}[t]
\caption{Experimental results under the RW policy (upper table) and UCB policy (lower table) under a policy $\pi_t = 0.7\pi^a_t+0.3\pi^u_t$ with $T=500$ and $N=1,000$, and the K-nearest neighbor regression. The method with the lowest MSE is in bold.} 
\begin{center}
\medskip
\label{tbl:appdx:exp_table3}
\scalebox{0.65}[0.65]{
\begin{tabular}{l|rr|rr|rr|rr|rr|rr|rr|rr}
\toprule
Datasets &  \multicolumn{2}{c|}{satimage}& \multicolumn{2}{c|}{pendigits}& \multicolumn{2}{c|}{mnist}& \multicolumn{2}{c|}{dna}& \multicolumn{2}{c|}{letter}& \multicolumn{2}{c|}{sensorless}& \multicolumn{2}{c|}{connect-4}& \multicolumn{2}{c}{covtype} \\
Metrics &      MSE &      SD &      MSE &      SD &      MSE &      SD &      MSE &      SD &      MSE &      SD &     MSE &     SD &      MSE &     SD &     MSE &     SD \\
\hline
FA3IPW &  {\bf 0.048} &  0.213 &  0.088 &  0.279 &  {\bf 0.080} &  0.283 &  {\bf 0.038} &  0.194 &  {\bf 0.229} &  0.457 &  {\bf 0.112} &  0.330 &  {\bf 0.029} &  0.169 &  0.083 &  0.287 \\
SFA3IPW &  0.055 &  0.222 &  0.090 &  0.282 &  0.098 &  0.307 &  0.043 &  0.207 &  0.296 &  0.521 &  0.118 &  0.343 &  0.038 &  0.195 &  0.082 &  0.287 \\
DM &  0.068 &  0.178 &  0.292 &  0.235 &  0.545 &  0.199 &  0.371 &  0.167 &  0.435 &  0.154 &  0.305 &  0.256 &  0.131 &  0.170 &  0.088 &  0.239 \\
AdaIPW &  0.089 &  0.280 &  0.143 &  0.356 &  0.114 &  0.330 &  0.080 &  0.282 &  0.255 &  0.474 &  0.157 &  0.394 &  0.077 &  0.242 &  0.127 &  0.354 \\
A2IPW &  {\bf 0.048} &  0.210 &  {\bf 0.086} &  0.277 &  {\bf 0.080} &  0.280 &  0.042 &  0.206 &  0.239 &  0.469 &  0.113 &  0.332 &  {\bf 0.029} &  0.170 &  {\bf 0.077} &  0.277 \\
FA2daIPW &  0.093 &  0.289 &  0.145 &  0.354 &  0.115 &  0.336 &  0.072 &  0.268 &  0.242 &  0.457 &  0.157 &  0.395 &  0.067 &  0.240 &  0.130 &  0.356 \\
\bottomrule
\end{tabular}
} 
\end{center}

\begin{center}
\scalebox{0.65}[0.65]{
\begin{tabular}{l|rr|rr|rr|rr|rr|rr|rr|rr}
\toprule
Datasets &  \multicolumn{2}{c|}{satimage}& \multicolumn{2}{c|}{pendigits}& \multicolumn{2}{c|}{mnist}& \multicolumn{2}{c|}{dna}& \multicolumn{2}{c|}{letter}& \multicolumn{2}{c|}{sensorless}& \multicolumn{2}{c|}{connect-4}& \multicolumn{2}{c}{covtype} \\
Metrics &      MSE &      SD &      MSE &      SD &      MSE &      SD &      MSE &      SD &      MSE &      SD &     MSE &     SD &      MSE &     SD &     MSE &     SD \\
\hline
FA3IPW &  0.123 &  0.350 &  0.136 &  0.364 &  {\bf 0.164} &  0.402 &  {\bf 0.051} &  0.225 &  0.172 &  0.415 &  0.102 &  0.306 &  5.501 &  2.345 &  0.076 &  0.276 \\
SFA3IPW &  0.081 &  0.277 &  0.113 &  0.335 &  0.210 &  0.457 &  0.057 &  0.238 &  0.217 &  0.463 &  {\bf 0.131} &  0.347 &  2.528 &  1.406 &  0.085 &  0.291 \\
DM &  {\bf 0.026} &  0.153 &  {\bf 0.088} &  0.144 &  0.509 &  0.151 &  0.376 &  0.165 &  0.435 &  0.185 &  0.277 &  0.255 &  0.131 &  0.163 &  0.090 &  0.216 \\
AdaIPW &  0.232 &  0.482 &  0.253 &  0.503 &  0.253 &  0.502 &  0.109 &  0.331 &  0.193 &  0.439 &  0.140 &  0.372 &  9.117 &  2.109 &  0.091 &  0.300 \\
A2IPW &  0.130 &  0.359 &  0.132 &  0.359 &  0.168 &  0.410 &  {\bf 0.051} &  0.225 &  {\bf 0.161} &  0.399 &  {\bf 0.100} &  0.305 &  2.805 &  1.661 &  {\bf 0.070} &  0.265 \\
FA2daIPW &  0.194 &  0.440 &  0.280 &  0.529 &  0.243 &  0.492 &  0.100 &  0.316 &  0.205 &  0.453 &  0.143 &  0.375 &  7.886 &  2.007 &  0.100 &  0.316 \\
\bottomrule
\end{tabular}
} 
\end{center}

\end{table*}

\begin{table*}[t]
\caption{Experimental results under the RW policy (upper table) and UCB policy (lower table) under a policy $\pi_t = 0.5\pi^a_t+0.5\pi^u_t$ with $T=1,000$ and $N=1,000$, and the Nadaraya-Watson regression. The method with the lowest MSE is in bold.} 
\begin{center}
\medskip
\label{tbl:appdx:exp_table4}
\scalebox{0.65}[0.65]{
\begin{tabular}{l|rr|rr|rr|rr|rr|rr|rr|rr}
\toprule
Datasets &  \multicolumn{2}{c|}{satimage}& \multicolumn{2}{c|}{pendigits}& \multicolumn{2}{c|}{mnist}& \multicolumn{2}{c|}{dna}& \multicolumn{2}{c|}{letter}& \multicolumn{2}{c|}{sensorless}& \multicolumn{2}{c|}{connect-4}& \multicolumn{2}{c}{covtype} \\
Metrics &      MSE &      SD &      MSE &      SD &      MSE &      SD &      MSE &      SD &      MSE &      SD &     MSE &     SD &      MSE &     SD &     MSE &     SD \\
\hline
FA3IPW &  0.043 &  0.207 &  {\bf 0.064} &  0.253 &  0.056 &  0.237 &  {\bf 0.023} &  0.151 &  0.105 &  0.296 &  0.062 &  0.237 &  0.032 &  0.150 &  {\bf 0.048} &  0.219 \\
SFA3IPW &  {\bf 0.040} &  0.199 &  0.066 &  0.257 &  {\bf 0.055} &  0.231 &  0.028 &  0.166 &  0.103 &  0.319 &  0.076 &  0.270 &  {\bf 0.031} &  0.152 &  0.052 &  0.228 \\
DM &  0.101 &  0.150 &  0.446 &  0.162 &  0.302 &  0.188 &  0.213 &  0.149 &  0.489 &  0.125 &  0.416 &  0.150 &  0.094 &  0.144 &  0.106 &  0.188 \\
AdaIPW &  0.075 &  0.273 &  0.108 &  0.327 &  0.067 &  0.255 &  0.047 &  0.217 &  0.102 &  0.303 &  0.072 &  0.259 &  0.160 &  0.237 &  0.070 &  0.261 \\
A2IPW &  0.043 &  0.208 &  0.065 &  0.255 &  0.057 &  0.238 &  {\bf 0.023} &  0.153 &  {\bf 0.098} &  0.290 &  {\bf 0.059} &  0.234 &  0.037 &  0.155 &  0.050 &  0.224 \\
FA2daIPW &  0.077 &  0.276 &  0.106 &  0.320 &  0.068 &  0.260 &  0.046 &  0.214 &  0.109 &  0.309 &  0.074 &  0.261 &  0.120 &  0.231 &  0.070 &  0.260 \\
\bottomrule
\end{tabular}
} 
\end{center}

\begin{center}
\scalebox{0.65}[0.65]{
\begin{tabular}{l|rr|rr|rr|rr|rr|rr|rr|rr}
\toprule
Datasets &  \multicolumn{2}{c|}{satimage}& \multicolumn{2}{c|}{pendigits}& \multicolumn{2}{c|}{mnist}& \multicolumn{2}{c|}{dna}& \multicolumn{2}{c|}{letter}& \multicolumn{2}{c|}{sensorless}& \multicolumn{2}{c|}{connect-4}& \multicolumn{2}{c}{covtype} \\
Metrics &      MSE &      SD &      MSE &      SD &      MSE &      SD &      MSE &      SD &      MSE &      SD &     MSE &     SD &      MSE &     SD &     MSE &     SD \\
\hline
FA3IPW &  0.041 &  0.203 &  0.096 &  0.310 &  0.155 &  0.384 &  {\bf 0.028} &  0.167 &  0.203 &  0.447 &  0.130 &  0.353 &  3.280 &  1.811 &  0.057 &  0.236 \\
SFA3IPW &  0.042 &  0.204 &  {\bf 0.092} &  0.303 &  {\bf 0.112} &  0.334 &  {\bf 0.028} &  0.166 &  0.214 &  0.462 &  {\bf 0.102} &  0.306 &  5.639 &  2.371 &  0.070 &  0.263 \\
DM &  {\bf 0.039} &  0.104 &  0.256 &  0.123 &  0.249 &  0.157 &  0.137 &  0.132 &  0.490 &  0.151 &  0.344 &  0.178 &  {\bf 0.115} &  0.174 &  0.063 &  0.135 \\
AdaIPW &  0.078 &  0.274 &  0.134 &  0.365 &  0.213 &  0.459 &  0.048 &  0.220 &  0.204 &  0.446 &  0.127 &  0.342 &  4.397 &  1.517 &  0.086 &  0.292 \\
A2IPW &  0.041 &  0.201 &  0.113 &  0.335 &  0.149 &  0.378 &  {\bf 0.028} &  0.167 &  {\bf 0.193} &  0.429 &  0.112 &  0.321 &  2.062 &  1.427 &  {\bf 0.051} &  0.223 \\
FA2daIPW &  0.083 &  0.284 &  0.138 &  0.371 &  0.225 &  0.470 &  0.057 &  0.239 &  0.217 &  0.465 &  0.151 &  0.382 &  3.792 &  1.355 &  0.096 &  0.307 \\
\bottomrule
\end{tabular}
} 
\end{center}

\end{table*}

\begin{table*}[t]
\caption{Experimental results under the RW policy (upper table) and UCB policy (lower table) under a policy $\pi_t = 0.5\pi^a_t+0.5\pi^u_t$ with $T=1,000$ and $N=1,000$, and the K-nearest neighbor regression. The method with the lowest MSE is in bold.} 
\begin{center}
\medskip
\label{tbl:appdx:exp_table5}
\scalebox{0.65}[0.65]{
\begin{tabular}{l|rr|rr|rr|rr|rr|rr|rr|rr}
\toprule
Datasets &  \multicolumn{2}{c|}{satimage}& \multicolumn{2}{c|}{pendigits}& \multicolumn{2}{c|}{mnist}& \multicolumn{2}{c|}{dna}& \multicolumn{2}{c|}{letter}& \multicolumn{2}{c|}{sensorless}& \multicolumn{2}{c|}{connect-4}& \multicolumn{2}{c}{covtype} \\
Metrics &      MSE &      SD &      MSE &      SD &      MSE &      SD &      MSE &      SD &      MSE &      SD &     MSE &     SD &      MSE &     SD &     MSE &     SD \\
\hline
FA3IPW &  {\bf 0.038} &  0.195 &  0.060 &  0.246 &  0.058 &  0.237 &  {\bf 0.025} &  0.158 &  {\bf 0.097} &  0.306 &  {\bf 0.055} &  0.234 &  {\bf 0.031} &  0.147 &  {\bf 0.044} &  0.206 \\
SFA3IPW &  0.040 &  0.199 &  {\bf 0.056} &  0.237 &  0.071 &  0.261 &  0.029 &  0.172 &  0.098 &  0.313 &  0.080 &  0.281 &  {\bf 0.031} &  0.156 &  0.045 &  0.210 \\
DM &  0.054 &  0.157 &  0.217 &  0.209 &  0.459 &  0.186 &  0.368 &  0.164 &  0.449 &  0.142 &  0.263 &  0.190 &  0.118 &  0.126 &  0.083 &  0.163 \\
AdaIPW &  0.053 &  0.230 &  0.098 &  0.313 &  0.091 &  0.301 &  0.054 &  0.232 &  0.107 &  0.326 &  0.071 &  0.267 &  0.137 &  0.185 &  0.074 &  0.251 \\
A2IPW &  {\bf 0.038} &  0.196 &  0.060 &  0.244 &  {\bf 0.056} &  0.233 &  {\bf 0.025} &  0.159 &  0.098 &  0.310 &  0.056 &  0.236 &  0.036 &  0.146 &  0.046 &  0.211 \\
FA2daIPW &  0.052 &  0.228 &  0.100 &  0.316 &  0.092 &  0.303 &  0.053 &  0.229 &  0.105 &  0.320 &  0.075 &  0.273 &  0.099 &  0.183 &  0.068 &  0.236 \\
\bottomrule
\end{tabular}
} 
\end{center}

\begin{center}
\scalebox{0.65}[0.65]{
\begin{tabular}{l|rr|rr|rr|rr|rr|rr|rr|rr}
\toprule
Datasets &  \multicolumn{2}{c|}{satimage}& \multicolumn{2}{c|}{pendigits}& \multicolumn{2}{c|}{mnist}& \multicolumn{2}{c|}{dna}& \multicolumn{2}{c|}{letter}& \multicolumn{2}{c|}{sensorless}& \multicolumn{2}{c|}{connect-4}& \multicolumn{2}{c}{covtype} \\
Metrics &      MSE &      SD &      MSE &      SD &      MSE &      SD &      MSE &      SD &      MSE &      SD &     MSE &     SD &      MSE &     SD &     MSE &     SD \\
\hline
FA3IPW &  0.059 &  0.243 &  0.075 &  0.272 &  0.106 &  0.324 &  0.033 &  0.179 &  {\bf 0.183} &  0.421 &  0.076 &  0.248 &   40.446 &   6.261 &  0.093 &  0.304 \\
SFA3IPW &  0.064 &  0.253 &  0.077 &  0.272 &  0.136 &  0.369 &  0.035 &  0.179 &  0.214 &  0.459 &  {\bf 0.068} &  0.255 &  106.595 &  10.183 &  0.071 &  0.263 \\
DM &  {\bf 0.016} &  0.114 &  {\bf 0.029} &  0.117 &  0.417 &  0.160 &  0.282 &  0.120 &  0.437 &  0.169 &  0.180 &  0.193 & {\bf 0.142} &   0.142 &  {\bf 0.036} &  0.163 \\
AdaIPW &  0.072 &  0.266 &  0.152 &  0.384 &  0.142 &  0.376 &  0.059 &  0.235 &  0.218 &  0.462 &  0.111 &  0.297 &   31.899 &   5.361 &  0.104 &  0.316 \\
A2IPW &  0.051 &  0.225 &  0.072 &  0.266 &  {\bf 0.103} &  0.319 &  {\bf 0.031} &  0.173 &  0.187 &  0.427 &  0.074 &  0.239 &   20.816 &   4.481 &  0.076 &  0.276 \\
FA2daIPW &  0.069 &  0.264 &  0.156 &  0.386 &  0.152 &  0.389 &  0.065 &  0.249 &  0.207 &  0.450 &  0.119 &  0.311 &   58.652 &   7.429 &  0.137 &  0.369 \\
\bottomrule
\end{tabular}
} 
\end{center}

\end{table*}

\end{document}